%% file: arxiv_version.tex
\documentclass[11pt]{article}

\usepackage[margin=1in]{geometry}

\usepackage[utf8]{inputenc} 
\usepackage[T1]{fontenc}    
\usepackage{lmodern}        
\usepackage[numbers, sort&compress]{natbib}
\usepackage{hyperref}       
\usepackage{url}            
\usepackage{booktabs}       
\usepackage{multirow}       
\usepackage{amsfonts}       
\usepackage{amsmath,amssymb,amsthm,mathtools} 
\usepackage{nicefrac}       
\usepackage{microtype}      
\usepackage{xcolor}         
\usepackage{graphicx}       
\usepackage{subcaption}     
\usepackage{cancel}
\usepackage{enumitem}       
\usepackage{wrapfig}        
\usepackage{etoc}           
\usepackage{authblk}        

\DeclareMathOperator{\tr}{tr}
\DeclareMathOperator*{\argmin}{arg\,min}
\DeclareMathOperator{\prox}{prox}

\newcommand{\RR}{\mathbb{R}}
\newcommand{\EE}{\mathbb{E}}
\newcommand{\PP}{\mathbb{P}}
\newcommand{\NN}{\mathcal{N}}

\newcommand{\I}{\mathbf{I}}
\newcommand{\C}{\mathbf{C}}
\newcommand{\x}{\mathbf{x}}
\newcommand{\bv}{\mathbf{v}}
\newcommand{\bmu}{\boldsymbol{\mu}}
\newcommand{\be}{\mathbf{e}}
\newcommand{\btheta}{\boldsymbol{\theta}}

\newtheorem{theorem}{Theorem}
\newtheorem{proposition}[theorem]{Proposition}
\newtheorem{lemma}[theorem]{Lemma}
\newtheorem{corollary}[theorem]{Corollary}
\newtheorem{assumption}{Assumption}
\newtheorem{remark}{Remark}

\definecolor{hadaspurple}{RGB}{128,0,128}

\title{When Stronger Triggers Backfire:\\ A High-Dimensional Theory of Backdoor Attacks}

\author[1]{D.G.M.~Flynn\textsuperscript{*}}
\author[2]{Hadas Yaron Goldhirsh\textsuperscript{*}}
\author[1]{Jonathan P. Keating\textsuperscript{\dag}}
\author[3]{Inbar Seroussi\textsuperscript{\dag}}

\affil[1]{Mathematical Institute, University of Oxford}
\affil[2]{School of Mathematical Science, Tel Aviv University}
\affil[3]{School of Mathematical Science and Computer Science, Tel Aviv University}
\affil[ ]{\textsuperscript{*}Equal primary contribution.\quad \textsuperscript{\dag}Equal senior contribution.}

\date{}

\begin{document}

\addtocontents{toc}{\protect\etocdepthtag.toc{mainpaper}}

\maketitle

\begin{abstract}
Backdoor poisoning attacks behave counter-intuitively in
high dimensions: stronger training triggers can help the defender. We study
regularised generalised linear models on Gaussian-mixture data in the
proportional regime ($p/n \to \kappa$), varying the training trigger strength
$\alpha$ against a fixed test trigger. Three phenomena emerge: (i) clean test
accuracy increases with $\alpha$; (ii) attack success peaks at a finite
$\alpha$ and then declines; and (iii) the most damaging trigger direction is
the minimum eigenvector of the data covariance. We prove all three results in
closed form for the squared loss, and extend (i) and (ii) to general convex GLM
losses via a Gaussian-proxy fixed-point system. We identify a finite-sample
noise floor proportional to $\kappa$ as the mechanism behind (i), invisible to
classical $n \gg p$ analysis. Experiments on CIFAR-10 and Gaussian surrogates match the theory closely; ResNet-18 experiments show the same phenomena beyond the convex setting.
\end{abstract}

\section{Introduction}
As machine learning is deployed in safety-critical domains, model security becomes a first-class concern: modern pipelines ingest data from the internet, and a tiny fraction of adversarial samples can compromise the model. Despite the canonical nature of backdoor poisoning attacks, a precise theoretical model that captures the high dimensional nature of modern datasets has remained elusive, with current work largely empirical or reliant on lower dimensional bounds.

We study \emph{backdoor} (trigger-based) poisoning~\citep{gu_badnets_2019}: the adversary injects training samples carrying a trigger and labelled with a target class, so that at test time the same trigger flips predictions while performing normally otherwise. In practice, the test-time trigger strength is constrained (e.g.\ the size of a sticker on a stop sign), while the training trigger strength $\alpha$ is a free choice for the attacker; stronger training triggers make poisoned samples easier to classify (potentially reducing their influence), weaker ones may leave too little imprint. We analyze this trade-off rigorously for regularised generalised linear models (GLMs) on high-dimensional Gaussian mixture data, working in the \emph{proportional regime} where the dimension of the data $p$ and the number of data points $n$ are of the same order. \textit{Is a stronger training trigger always a stronger attack?}

\paragraph{Motivating evidence from deep networks}
Figure~\ref{fig:resnet_intro} shows the clean test accuracy and attack success rate of a ResNet-18 on CIFAR-10. Three striking patterns emerge: (i) clean accuracy \emph{increases} with $\alpha$, (ii) the attack success rate peaks at a finite $\alpha$ and then declines, and (iii) the attack is most effective when the trigger aligns with low-variance directions of the data. 
\textit{What underlying mechanisms of standard training give rise to these phenomena?
}

\begin{figure}[ht]
    \includegraphics[width=\textwidth]{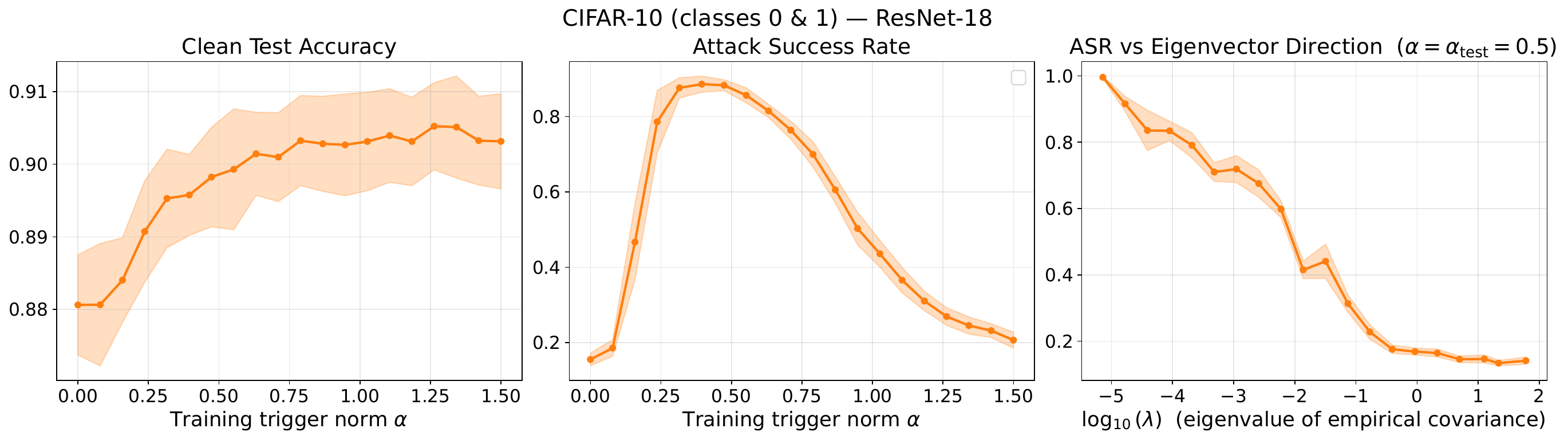}
    \caption{Observed Phenomena in Backdoor Poisoning}
    \smallskip
    \small
    ResNet-18 on CIFAR-10 (classes 0 \& 1); $\phi=0.05$; $\alpha_{\text{test}} = 0.5$, averaged over 25 runs. Left/middle: clean accuracy and attack success rate vs.\ \emph{training} trigger strength $\alpha$, with a fixed $2{\times}2$-pixel corner trigger. Right: at $\alpha=0.5$, the trigger direction varies over eigenvectors of the empirical covariance (x-axis: log eigenvalue).
    \label{fig:resnet_intro}
\end{figure}

\subsection{Our contributions}
We give the first high-dimensional theoretical characterization of the trigger's impact on data poisoning, including a rigorous account of the emergence of the behaviour also observed in deep learning systems described in Figure~\ref{fig:resnet_intro}. We model backdoor-poisoned data as a binary Gaussian mixture with general covariance $\C$ and analyse regularised GLMs in two settings: Empirical risk minimisation (ERM) in the proportional regime and population-risk minimisation in the information limit ($n \gg p$). The central quantities are the \emph{benign alignment} and \emph{trigger alignment}, which govern clean accuracy and attack success rate respectively (Section~\ref{sec:setting}, Theorem~\ref{thm:loureiro}).

We prove the three phenomena of Figure~\ref{fig:resnet_intro} across three settings. \textbf{(i) Squared loss, ERM} (Section~\ref{sec:exact}): closed-form alignments yield benign monotonicity, a unique finite peak in trigger alignment, and the minimum eigenvector of $\C$ as the optimal attack direction (Proposition~\ref{prop:lin_peak_general}, Corollary~\ref{cor:lin_clean_general}, and Proposition~\ref{cor:lin_min_eig}). \textbf{(ii) Convex losses, proportional regime} (Section~\ref{sec:erm_logistic}): the trigger alignment peaks at a finite $\alpha^\star$ with explicit decay rates (Proposition~\ref{prop:alignment_peaks}). \textbf{(iii) Convex losses, information limit} (Section~\ref{sec:pop_logistic}): the minimiser converges to its unpoisoned value as $\alpha \to \infty$ (Proposition~\ref{prop:pop_convergence}), and a one-step monotonicity argument (Proposition~\ref{prop:one_step}) gives the benign-alignment ordering.

We further show that a finite-sample noise floor proportional to $\kappa$ drives the increase in clean accuracy, an effect invisible to the information limit (Section~\ref{sec:comparison}).
All results are validated on logistic regression with real and Gaussian surrogate data, and all three phenomena persist in a ResNet-18 (Figure~\ref{fig:resnet_intro}), showing universality well beyond our convex theoretical setting.

Our work contributes a rigorous, quantitative analysis of backdoor poisoning, where prior work has largely relied on empirical evidence or low-dimensional bounds. From a high-dimensional statistics perspective, we introduce a tractable adversarial data model and show that it exhibits qualitatively new behavior not captured by classical $n \gg p$ asymptotics.

\subsection{Related work}

\paragraph{Backdoor attacks.}
Trigger-based poisoning was introduced by~\citet{gu_badnets_2019}; triggers have grown stealthier through reflections~\citep{liu_reflection_2020}, physical deployment~\citep{li_backdoor_2021}, label-consistent or warping constructions~\citep{turner_label-consistent_2019,nguyen_wanet_2021}, data-efficient variants~\citep{xia_data-efficient_2022,souly_poisoning_2025}, and low-rank-activated triggers~\citep{evansTheoryMinimalWeight2026}; the same mechanism underpins watermarking~\citep{adi_turning_2018} and unlearning verification~\citep{pawelczyk_machine_2025} (see~\citet{li_backdoor_2024} for a survey of attacks and defenses). Theory has so far relied on empirical study~\citep{gu_badnets_2019}, lower-dimensional bounds~\citep{lu_exploring_2023}, or detectability--efficacy trade-offs~\citep{granziolSafetyEfficacyTradeRobustness2026}; \citet{flynn_linear_2026} took a first high-dimensional step for ridge regression with isotropic covariates, which we extend to general convex losses and general covariance.

\paragraph{High-dimensional proportional regime.}
The proportional regime ($p/n \to \kappa$) has been used to explain modern learning phenomena such as double descent and benign overfitting~\citep{hastie2019surprises, belkin2019, bartlett2020benign, bartlett21deep}. 
Sharp asymptotic characterizations in this regime have been developed using tools such as approximate message passing (AMP), the convex Gaussian min--max theorem (CGMT), and leave-one-out techniques~\citep{stojnic_framework_2013,thrampoulidis_regularized_2015,rangan_generalized_2012,javanmard_state_2012,donoho_high_2013,el_karoui_robust_2013,mai_high_2020}. These methods yield precise results for ridge and logistic regression~\citep{dobriban2018high,sur2019modern,montanari2025generalization}, Gaussian mixture classification~\citep{mai_high_2020,deng_model_2020,loureiro_learning_2021}, and related teacher--student models~\citep{liang_precise_2022,montanari_generalization_2023}, as well as Bayes-optimal inference ~\citep{barbier_optimal_2019,seroussi_lower_2022}. %
A complementary line of work studies the dynamics of stochastic gradient descent via high-dimensional ODE/SDE limits~\citep{paquette24homogenization, collins2024hitting, arous2022highdimensional}, providing insight into high-dimensional optimisation. %
Closest to our setting is~\citet{barnfield_high-dimensional_2025}, which studies sparse signal detection in high dimensions; in contrast, we analyse adversarially injected triggers in regularised GLMs.
\section{Setting and notation}
\label{sec:setting}


\paragraph{Clean data.}
We consider a classification problem where each sample $(\x_i, y_i)\in \mathbb{R}^p\times \{\pm 1\}$ is drawn independently from a binary Gaussian mixture
with $\PP(y_i = +1) = \tfrac{1}{2}$, and the
feature vector $\x_i \mid y_i \sim \NN(y_i \bmu, \C)$, where
$\bmu \in \RR^p$ is the class mean and
$\C \in \RR^{p \times p}$ is a shared
covariance matrix.

\paragraph{Poisoning model.}
An adversary selects a fraction $\phi \in (0, \tfrac{1}{2})$ of the
full training set, drawn exclusively from the samples with
$y_i = -1$, and applies a \emph{backdoor attack}: for each selected
sample, the trigger $\alpha \bv$ is added to the feature vector and
the label is flipped to $+1$. Here $\bv \in \RR^p$ is a fixed
trigger direction with $\|\bv\| = 1$, $\alpha \geq 0$ controls the
trigger strength, and samples with $y_i = +1$ are never modified. We assume $\alpha \geq 0$, as otherwise we may just consider $-\bv$.

The losses we consider depend on features and labels only through $y_i \x_i$\footnote{Equivalently, $\tilde L(x, +1) = \tilde L(-x, -1)$; we write $L(x) := \tilde L(x, +1)$.}, so we absorb labels into features as $\mathbf{z}_i := y_i \x_i$.
After poisoning, the absorbed data follows a two-component Gaussian mixture with shared covariance $\C$:
\begin{equation}
\label{eq:mixture}
\mathbf{z}_i \mid K = c \;\sim\; \NN(\bmu_c,\, \C),
\qquad
\bmu_1 = \bmu,\quad
\bmu_2 = \alpha\bv - \bmu,\quad
\pi_1 = 1-\phi,\quad
\pi_2 = \phi.
\end{equation}

\paragraph{Notation.}
We use boldface for matrices and vectors, with capitals for matrices. For a matrix $\mathbf{A}$, $\|\mathbf{A}\|$ denotes its operator norm. The complexity notations  $O(\cdot)$, and $o(\cdot)$ are understood for large data size $n$ and input dimension $p$, while the notation $O_\alpha(\cdot),  o_\alpha(\cdot)$ is intended for sufficiently small $\alpha^{-1}$. 
We use $\hat\btheta$ for the learned parameters (ERM or population minimiser, depending on context), and $\tilde\btheta$ for the Gaussian proxy from Theorem~\ref{thm:loureiro}.

\paragraph{Empirical risk minimisation (proportional regime)}
\label{sec:erm}

Given $n$ training samples $(\x_i, y_i)_{i=1}^n$ and parameter vector $\btheta\in \mathbb{R}^p$; a convex loss $L$, and regularisation $\lambda > 0$, we consider the estimator
\begin{equation}
\label{eq:erm}
\hat\btheta \;=\; \argmin_{\btheta \in \RR^p} \;
\big\{\mathcal{L}_{n}(\btheta)
\;:=\;
 \frac{1}{n}\sum_{i=1}^{n} L\!\bigl(y_i\,
\x_i^\top \btheta \bigr)
\;+\; \frac{\lambda}{2}\|\btheta\|^2\big\}.
\end{equation}
Canonical choices are the logistic loss $L(t) = \log(1 + e^{-t})$ and the squared loss $L(t) = \tfrac{1}{2}(1-t)^2$. 

\begin{assumption}[Proportional asymptotics]
\label{ass:asymptotics}
The dimension $p$ and sample size $n$ grow jointly such that
$p/n \to \kappa \in (0,\infty)$.
The covariance is bounded, $\|\C\| = O(1)$, and the class
means satisfy $\|\bmu_c\| = O(1)$ for each $c \in \{1,2\}$.
The loss $L: \RR \rightarrow \RR$ is strictly convex.
\end{assumption}

\paragraph{Population risk minimisation (information limit)}
\label{sec:population}

In the information limit $\kappa \to 0$ (or $n \gg p$), this is equivalent to minimising the \emph{population risk} (Appendix~\ref{app:erm_pop_bound}):
\begin{equation}
\label{eq:pop_loss}
\mathcal{L}_{\mathrm{pop}}(\btheta;\alpha)
\;:=\;
(1-\phi)\,\EE_{\x \sim \NN(\bmu,\C)}\!\bigl[L(\btheta^\top\x)\bigr]
\;+\;
\phi\,\EE_{\x \sim \NN(\alpha\bv-\bmu,\C)}\!\bigl[L(\btheta^\top\x)\bigr]
\;+\;
\frac{\lambda}{2}\|\btheta\|^2,
\end{equation}
with unique minimiser $\btheta(\alpha)$ (by $\lambda$-strong convexity). 
\subsection{Gaussian proxy for ERM}
\label{sec:det_equiv}

Our ERM-regime analysis builds on the following characterisation of the
high-dimensional minimiser, which is a consequence of the general
results of \citet{loureiro_learning_2021}.

\begin{theorem}[Gaussian proxy {\citep{loureiro_learning_2021}}]
\label{thm:loureiro}
Denote $\mathbf{R}(\lambda,\tau):=(\lambda \I+\tau \C)^{-1}$. Under Assumptions~\ref{ass:asymptotics} and the mixture model~\eqref{eq:mixture}, there exists $\tilde\btheta$ such that for pseudo-Lipschitz $\phi:\RR^p\to\RR$ of finite order, $\phi(\hat\btheta)\to_{p,n\to\infty}\EE\phi(\tilde\btheta)$ in probability, where
\[
\mathbf{R}(\lambda,\tau)^{-1}\tilde\btheta \sim \NN\!\bigl(\eta_1\bmu_1+\eta_2\bmu_2,\; \tfrac{\gamma}{n}\C\bigr).
\]
The scalars $(\tau,\gamma,\delta,\eta_1,\eta_2)\in\RR^5$ solve a self-consistent system. Let $f(x)=-L'(\prox_{\delta L}(x))$ and $r_c\sim\NN(M_c,\sigma^2)$, with $r_K\sim r_c$ conditional on $K=c$ with probability $\pi_c$. Then $\tau=\EE[-f'(r_K)]$, $\gamma=\EE[f^2(r_K)]$, and $\eta_c=\pi_c\,\EE[f(r_c)]$ for $c=1,2$, and
\[
\delta=\tfrac{1}{n}\tr[\C \mathbf{R}(\lambda,\tau)],\quad 
M_c=\bmu_c^\top \mathbf{R}(\lambda,\tau)(\eta_1\bmu_1+\eta_2\bmu_2),\; c=1,2,
\]
\[
\sigma^2=(\eta_1\bmu_1+\eta_2\bmu_2)^\top \mathbf{R}(\lambda,\tau)^2 \C (\eta_1\bmu_1+\eta_2\bmu_2)
+\tfrac{\gamma}{n}\tr\!\bigl[\mathbf{R}(\lambda,\tau)^2\C^2\bigr].
\]
The scalars $M_c$ and $\sigma^2$ are the asymptotic mean and variance of $\bmu_c^\top\tilde\btheta$ and $\x^\top\tilde\btheta$ for a test point $\x$.
\end{theorem}

The key quantities in our analysis are the expected benign and trigger alignments, defined as
\[
h_\mu(\alpha) := \mathbb{E}[\bmu^\top \tilde{\btheta}], \quad
h_v(\alpha) := \mathbb{E}[\bv^\top \tilde{\btheta}].
\]
For \(\alpha > 0\), we have \(h_v(\alpha) = (M_1 + M_2)/\alpha\).

Our results use the following genericity condition on the trigger direction relative to the data's mean. In high dimensions, almost all pairs of vectors are nearly orthogonal: for $\bv$ uniform on the sphere, $\bv^\top R(\lambda, \tau)\bmu = O(p^{-1/2})$ with probability tending to one. There is no a priori reason for a trigger to have substantial overlap with the mean direction. 
Moreover, standard backdoor trigger constructions (e.g., localised patches) are not designed to align with class-specific structure and thus typically have negligible correlation with the class mean.
\begin{assumption}[Trigger orthogonality]
\label{ass:orthogonality}
The trigger direction $\bv$ is asymptotically orthogonal to the
Krylov subspace of $\C$ generated by $\bmu$:
$\bv^\top \C^k \bmu = o(1)$ for each $k = 0, 1, 2, \ldots$
Equivalently, $\bv^\top
R(\lambda, \tau)\bmu = o(1)$
for every fixed $\tau \geq 0$ and $\lambda>0$.
\end{assumption}

Section \ref{sec:erm_logistic} requires mild regularity of the loss, satisfied by the logistic and exponential losses. 

\begin{assumption}[Loss regularity]
\label{ass:loss}
The loss $L$ is non-negative, convex, and strictly decreasing.
It has an exponentially bounded derivative for $x < 0$ and
sufficient decay for $x > 0$:
\[
|L'(x)| \leq e^{C_1(|x|+1)} \;\text{for } x < 0,
\qquad
|L'(x)| \leq \frac{C_2}{x^{1+\epsilon}+1} \;\text{for } x > 0
\text{ and some } \epsilon > 0.
\]
\end{assumption} 


\section{Main results}
\label{sec:main_results}

We present three phenomena, each stated informally here with numerical support on CIFAR-10 and Gaussian surrogates; formal results appear in Sections~\ref{sec:exact}--\ref{sec:pop_logistic}. Throughout, the attack success rate is evaluated on held-out samples with a fixed test trigger $\alpha_{\mathrm{test}}\bv$.

\begin{figure}[ht]
    \includegraphics[width=\textwidth]{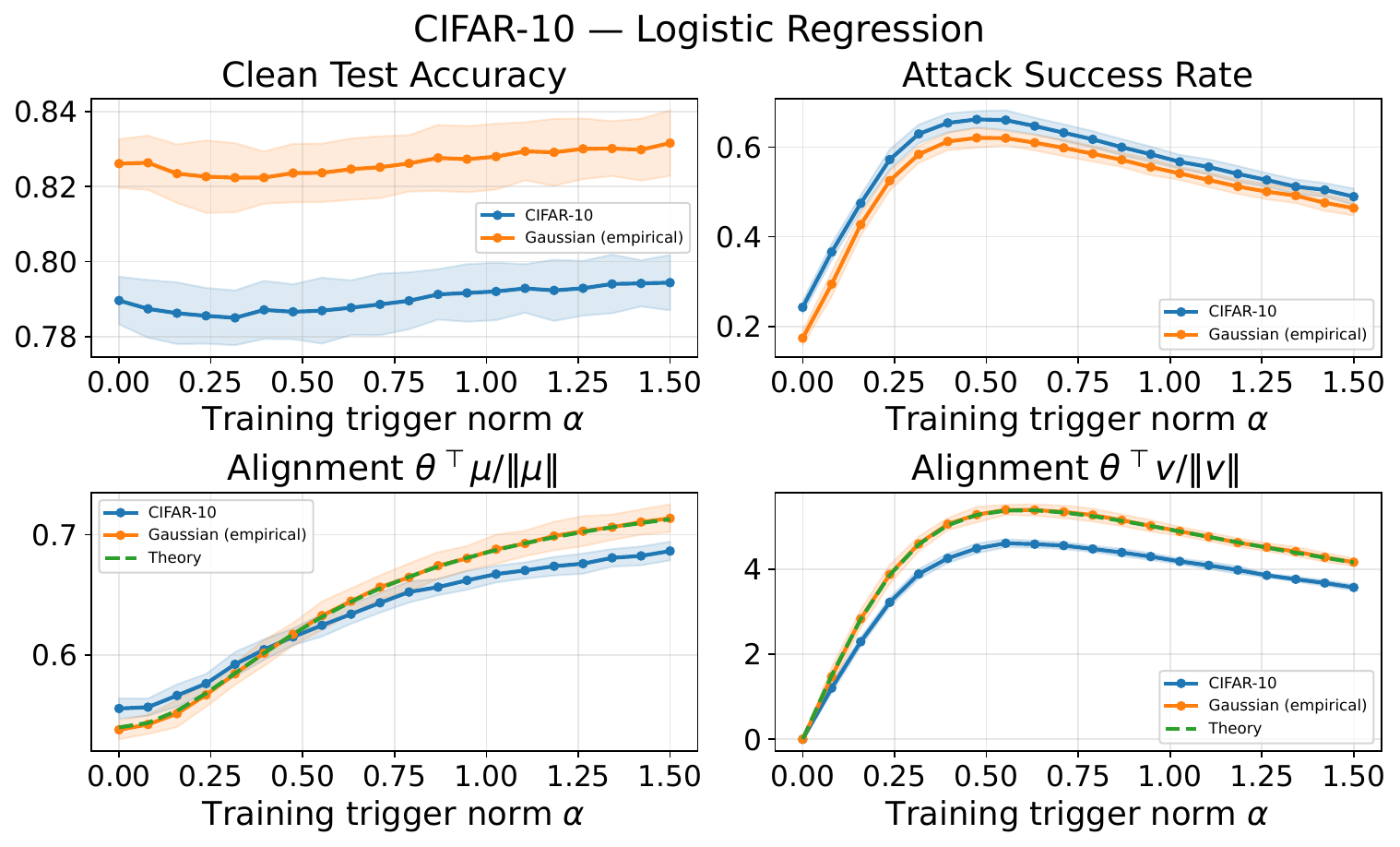}
    \caption{Real Data vs Theoretical Predictions}
    \smallskip
    \small A plot of CIFAR-10 (classes 0 \& 1) for logistic regression on real data
    (blue) and Gaussian surrogates (orange) compared against theoretical predictions (dashed) obtained by solving the fixed point equation in Theorem \ref{thm:loureiro} numerically. Here $\phi=0.05$; $\alpha_{\text{test}} = 0.5$. The Gaussian surrogates act as a proxy for CIFAR-10, which lie within our theoretical assumptions, further experimental details are in Appendix \ref{app:experimental_details}.
    \label{fig:clean_acc_cifar}
\end{figure}

We also define the \emph{benign-only} objective, obtained by dropping
the poisoned term from~\eqref{eq:pop_loss} and with unique minimiser $\btheta_{\mathrm{ben}}$:
\begin{equation}
\label{eq:ben_loss}
\mathcal{L}_{\mathrm{ben}}(\btheta)
\;:=\;
(1-\phi)\,\EE_{\x \sim \NN(\bmu,\C)}\!\bigl[L(\btheta^\top\x)\bigr]
\;+\;
\frac{\lambda}{2}\|\btheta\|^2.
\end{equation}

\paragraph{Clean accuracy  \textit{increases} with trigger strength}
\label{sec:clean_acc}

For fixed poisoning fraction $\phi$, increasing the trigger strength $\alpha$ counter-intuitively \emph{improves} clean test accuracy. The mechanism: as $\alpha$ grows, poisoned samples are pushed further into the $+1$ region and become easier to classify, so the model ``wastes'' less capacity on them and concentrates on the clean distribution.
\emph{Formally:} we prove this for the ERM squared loss (Corollary~\ref{cor:lin_clean_general}). For general losses in the information limit, the poisoned minimiser converges to its unpoisoned counterpart, $\btheta(\alpha) \to \btheta_{\mathrm{ben}}$ as $\alpha \to \infty$ (Proposition~\ref{prop:pop_convergence}); when $\bmu$ is an eigenvector of $\C$, the unpoisoned model has strictly higher benign alignment, $\bmu^\top\btheta(0) < \bmu^\top\btheta_{\mathrm{ben}}$, and a one-step monotonicity argument (Proposition~\ref{prop:one_step}) gives the general ordering.

Figure~\ref{fig:clean_acc_cifar} confirms this on CIFAR-10: logistic regression's clean accuracy increases monotonically with $\alpha$, and the Gaussian surrogate closely tracks the real data.


\paragraph{Alignment of the learned direction peaks with trigger strength}
\label{sec:alignment_peak}

The attack success rate is controlled by the trigger alignment $\bv^\top\tilde\btheta$. Rather than growing without bound in $\alpha$, the trigger alignment is maximised at a finite $\alpha^\star$ and decays thereafter: for large $\alpha$ the poisoned class is so well-separated that the loss provides vanishing gradient in the trigger direction (since $L'(x) \to 0$ as $x \to +\infty$ under Assumption~\ref{ass:loss}). Formally, in the proportional regime $\bv^\top\tilde\btheta \to 0$ as $\alpha \to \infty$, at rate $O(\alpha^{-\epsilon/(2+\epsilon)})$ for polynomial-tail losses and $O(\log\alpha / \alpha)$ for exponential-tail losses (e.g.\ logistic). Figure~\ref{fig:clean_acc_cifar} shows this peaking on CIFAR-10 and its Gaussian surrogate, and Figure~\ref{fig:resnet_intro} shows it persists in ResNet-18.


\paragraph{Trigger direction and the minimum eigendirection of the covariance}
\label{sec:trigger_eigen}

For a fixed trigger budget $\alpha$, the attack is most effective along low-variance directions of the data: the classifier has little signal there to resist the perturbation, so the decision boundary is easiest to shift. This singles out the eigenvector of $\C$ with the smallest eigenvalue as the optimal trigger direction. 
\emph{Formally:} for the squared loss the trigger alignment is monotone in the resolvent quadratic form $\bv^\top \mathbf{R}(\lambda,\tau)\bv$, and is maximised precisely along the minimum eigenvector of~$\C$ (Corollary~\ref{cor:lin_min_eig}).
Numerically this extends to the logistic loss on CIFAR-10 and to ResNet-18.

\subsection{Exact results ERM least squares}
\label{sec:exact}

For the squared loss \(L(t)=\frac12(1-t)^2\), the Gaussian-proxy fixed point
reduces to an explicit ridge-type calculation. In particular, for fixed
\((\lambda,\C)\), the effective scalar \(\tau\) is independent of the trigger
strength \(\alpha\), and the proxy expectation \(\EE[\tilde{\btheta}]\) lies in the
resolvent-weighted span
\[
\operatorname{span}\{\mathbf{R}(\lambda,\tau)\bmu,\;\mathbf{R}(\lambda,\tau)\bv\}.
\]
The full analysis, including arbitrary covariance projections, is given in Proposition~\ref{prop:lin_exact_projections} in Appendix~\ref{app:linear_regression}. Here we highlight the three main consequences: a finite peak in the trigger projection, monotonic growth of the benign projection (at leading order), and the spectral dependence on \(\C\).



\subsubsection{Finite peak of the trigger projection}

We define $g_{\mu\mu}:=\bmu^\top \mathbf{R}(\lambda,\tau)\bmu$, $g_{\mu v}:=\bmu^\top \mathbf{R}(\lambda,\tau)\bv$, $g_{vv}:=\bv^\top \mathbf{R}(\lambda,\tau)\bv$. 

\begin{proposition}[Finite peak under generic trigger orthogonality]
\label{prop:lin_peak_general}
Assume that \(0<\phi<1/2\), that \(\bmu\) and \(\bv\) are linearly independent,
and that Assumption~\ref{ass:orthogonality} holds so that $g_{\mu v}=o(1)$.
Consequently, uniformly for \(\alpha\) in any fixed compact subset of
\([0,\infty)\),
\begin{equation}
\label{eq:lin_hv_peak_compact}
h_v(\alpha)
=
\frac{A_v\alpha}{B+C\alpha^2}
+o(1),
\end{equation}
where $A_v:=\tau\phi g_{vv}\bigl(1+2\tau(1-\phi)g_{\mu\mu}\bigr)$, $B:=1+\tau g_{\mu\mu}$, and $C:=\tau\phi g_{vv}\bigl(1+\tau(1-\phi)g_{\mu\mu}\bigr)$.
In particular, for every fixed \(\alpha>0\), \(h_v(\alpha)>0\) eventually.

Moreover, the exact trigger projection has a unique finite maximizer
\(\alpha_{\star,\varepsilon}\), where \(\varepsilon:=g_{\mu v}\), and
\begin{equation}
\label{eq:lin_alpha_star_compact}
\alpha_{\star,\varepsilon}^2
=
\frac{B}{C}
+o(1)
=
\frac{1+\tau g_{\mu\mu}}
{\tau\phi g_{vv}\bigl(1+\tau(1-\phi)g_{\mu\mu}\bigr)}
+o(1).
\end{equation}
Thus \(h_v\) is strictly increasing on
\((0,\alpha_{\star,\varepsilon})\) and strictly decreasing on
\((\alpha_{\star,\varepsilon},\infty)\).
\end{proposition}


The rational form~\eqref{eq:lin_hv_peak_compact} reveals the mechanism: a linear numerator and quadratic denominator force \(h_v(\alpha)\) to decay at large \(\alpha\), so the trigger direction ultimately loses influence even under squared loss.

\subsubsection{Benign projection is monotone to leading order}

Similarly, we get a leading-order form for the benign projection, with quadratic-in-$\alpha$ numerator

\begin{corollary}[Leading-order benign projection under generic trigger orthogonality]
\label{cor:lin_clean_general}
Under the assumptions of Proposition~\ref{prop:lin_peak_general}, uniformly
for \(\alpha\) in any fixed compact subset of \([0,\infty)\),
\[
h_\mu(\alpha)
=
\frac{A_\mu^{(0)}+A_\mu^{(2)}\alpha^2}{B+C\alpha^2}
+o(1),
\]
with \(B,C\) as in Proposition~\ref{prop:lin_peak_general} and \(A_\mu^{(0)}:=\tau(1-2\phi)g_{\mu\mu}\), \(A_\mu^{(2)}:=\tau^2\phi(1-\phi)g_{\mu\mu}g_{vv}\). The leading-order curve is strictly increasing for every \(\alpha>0\); equivalently, the exact derivative \(h_\mu'(\alpha)>0\) eventually uniformly on every compact \(K\subset(0,\infty)\).
\end{corollary}
The translation from benign projection to clean accuracy is discussed in Section~\ref{sec:comparison}.

\begin{figure*}[t]
    \centering
    \begin{subfigure}[t]{0.485\textwidth}
        \centering
        \includegraphics[width=\linewidth]{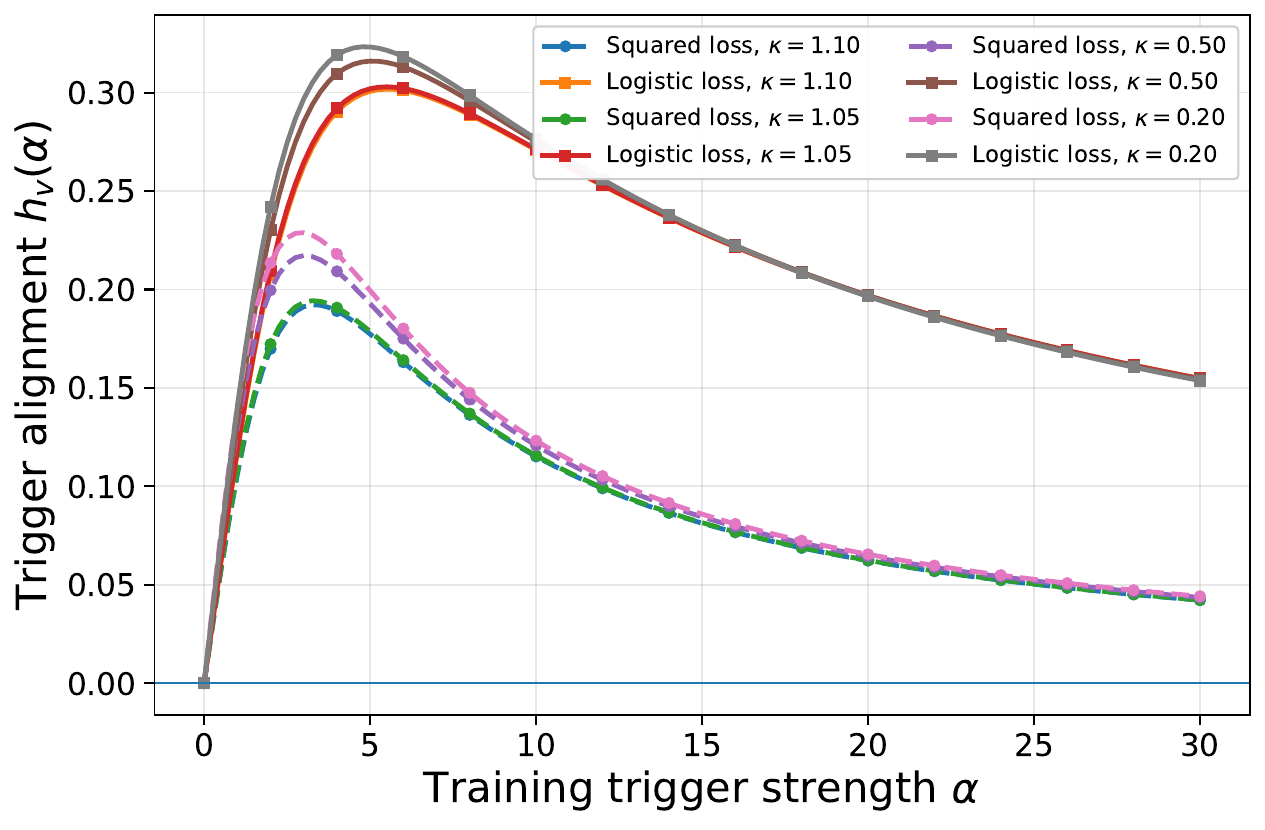}
        \caption{Trigger alignment \(h_v(\alpha)\).}
    \end{subfigure}
    \hfill
    \begin{subfigure}[t]{0.485\textwidth}
        \centering
        \includegraphics[width=\linewidth]{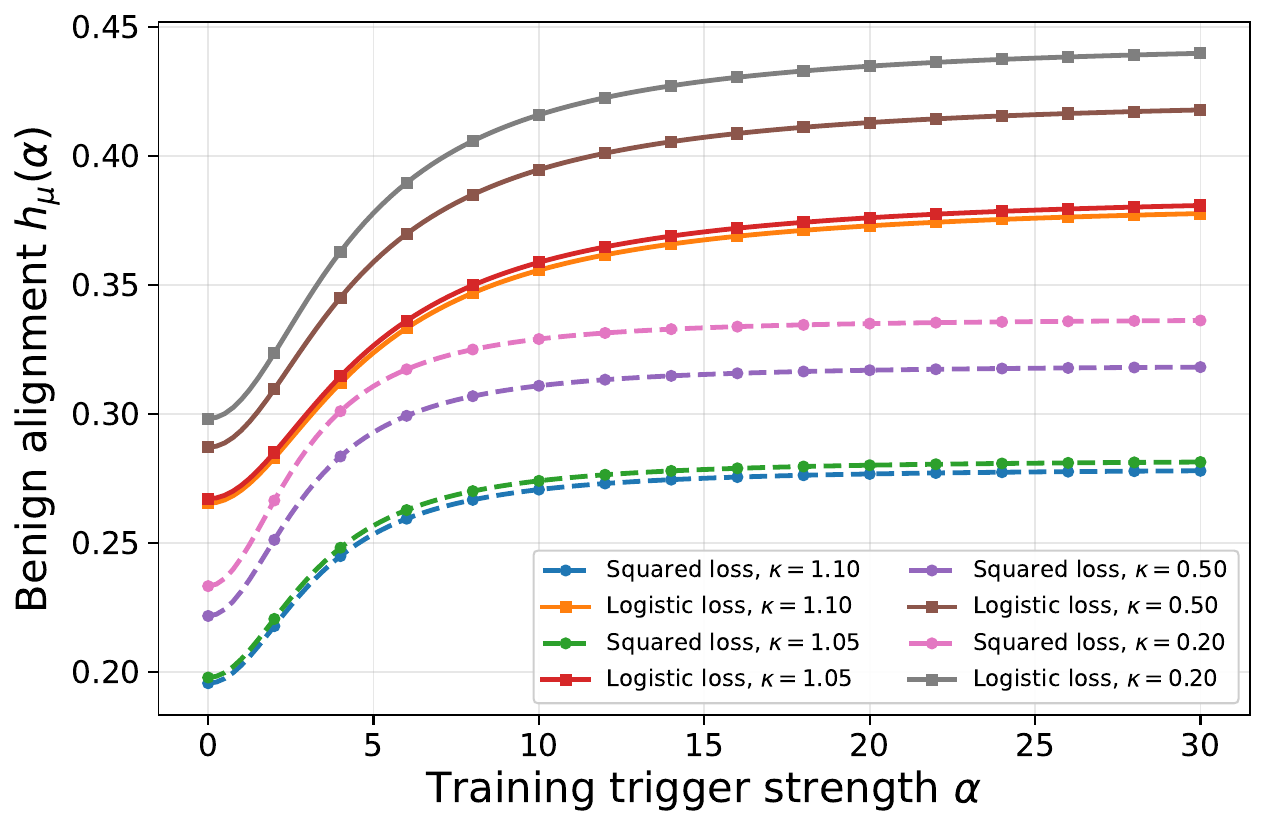}
        \caption{Benign alignment \(h_\mu(\alpha)\).}
    \end{subfigure}
\caption{Projection-level fixed-point predictions in the isotropic setting
\(\C=\I\) across aspect ratios \(\kappa=p/n\), including the overparameterized regime \(\kappa>1\).
Dashed curves show squared-loss predictions, solid curves logistic-loss.
Left: the trigger alignment \(h_v(\alpha)\) peaks at a finite \(\alpha\).
Right: the benign alignment \(h_\mu(\alpha)\) increases with \(\alpha\).
These behaviors persist for both \(\kappa<1\) and \(\kappa>1\).}
    \label{fig:lin_peak_benign}
\end{figure*}

Figure~\ref{fig:lin_peak_benign} confirms both phenomena across aspect ratios
\(\kappa=p/n\), including the overparameterized regime \(\kappa>1\), and shows
that they persist for the logistic loss beyond the squared-loss
setting.

\subsubsection{Eigenvector specialization}
To gain intuition, we restrict attention to a simplified setting in which the clean and trigger directions align with orthogonal eigenvectors of \(\C\). In this case, Assumption~\ref{ass:orthogonality} holds exactly, and the alignment quantities admit closed-form expressions.
\begin{corollary}[Eigenvector specialisation]
\label{cor:lin_peak_eigen}
Assume  $\bmu^\top\bv=0$, $\C\bmu=s_\mu^2\bmu$, $\C\bv=s_v^2\bv$. Define $A_\mu:=\lambda+\tau s_\mu^2+\tau\|\bmu\|^2$, $B_v:=\lambda+\tau s_v^2$, $P_\mu:=\lambda+\tau s_\mu^2+\tau(1-\phi)\|\bmu\|^2$, $Q_\mu:=\lambda+\tau s_\mu^2+2\tau(1-\phi)\|\bmu\|^2$, and $D_{\mathrm{eig}}(\alpha):=A_\mu B_v+\tau\phi P_\mu\alpha^2$. Then
\begin{equation}
\label{eq:lin_ab_eigen_compact}
h_\mu(\alpha)
=
\|\bmu\|^2
\frac{\tau(1-2\phi)B_v+\tau^2\phi(1-\phi)\alpha^2}
{D_{\mathrm{eig}}(\alpha)},
\qquad
h_v(\alpha)
=
\frac{\tau\phi Q_\mu\alpha}
{D_{\mathrm{eig}}(\alpha)}.
\end{equation}
Consequently, \(h_\mu(\alpha)\) is strictly increasing for \(\alpha>0\),
while \(h_v(\alpha)>0\) for every \(\alpha>0\) and has a unique finite
maximizer at
\begin{equation}
\label{eq:lin_alpha_star_eigen_compact}
(\alpha_\star^{\mathrm{eig}})^2
=
(A_\mu B_v)/(\tau\phi P_\mu).
\end{equation}
\end{corollary}

This is the \emph{low-variance effect}: increasing \(s_v^2\) attenuates \(h_v(\alpha)\) and pushes the maximizer \(\alpha_\star^{\mathrm{eig}}\) to larger values, so the model is most sensitive along directions where the data vary least.

See Appendix~\ref{app:linear_regression} for further analysis, including how the poisoning fraction \(\phi\) modulates the strength and location of the peak (Figure~\ref{fig:lin_phi_sweep_appendix}).



\begin{proposition}[Minimum-eigenvalue trigger directions]
\label{cor:lin_min_eig}
Under the assumptions of Proposition~\ref{prop:lin_peak_general}, particularly the asymptotic trigger orthogonality, Assumption~\ref{ass:orthogonality}, then
the leading-order trigger projection \eqref{eq:lin_hv_peak_compact} depends on
\(\bv\) only through \(g_{vv}=\bv^\top \mathbf{R}(\lambda,\tau)\bv\), and is strictly
increasing in \(g_{vv}\) at every fixed \(\alpha>0\). Consequently, varying
\(\bv\) over the unit sphere, the leading-order maximizer of \(h_v(\alpha)\)
is the eigenvector of \(\C\) associated with its smallest eigenvalue
\(s_{\min}^2\), attaining \(g_{vv}=1/(\lambda+\tau s_{\min}^2)\).
\end{proposition}

\begin{proof}
In \eqref{eq:lin_hv_peak_compact} of Proposition~\ref{prop:lin_peak_general},
$A_v$ and $C$ are linear in $g_{vv}$ while $B$ is independent of $\bv$, from which
monotonicity in $g_{vv}$ follows by direct calculation. The unit-norm maximum of
$g_{vv}=\bv^\top \mathbf{R}(\lambda,\tau)\bv$ is the largest eigenvalue of
$\mathbf{R}(\lambda,\tau)=(\lambda\I+\tau\C)^{-1}$.
\end{proof}

\begin{figure*}[t]
    \centering
    \includegraphics[width=\textwidth]{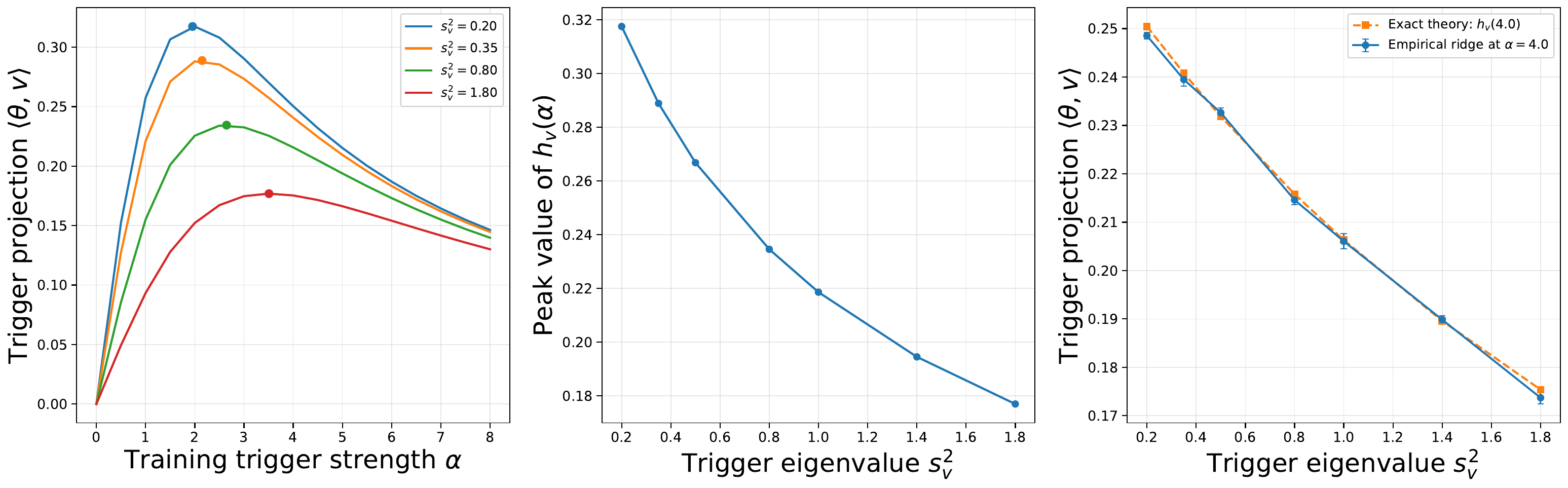}
    \caption{Eigenvector specialization. Left: exact trigger-projection curves \(h_v(\alpha)\)
    for several trigger eigenvalues \(s_v^2\), with the peak marked on each curve (coloured dots). Middle: the
    peak value \(\max_{\alpha} h_v(\alpha)\) decreases as \(s_v^2\) increases. Right: at fixed
    \(\alpha\), exact theory and empirical ridge estimates agree and both show that lower-variance
    trigger eigendirections induce larger trigger projections. Here, $\C$ is a rank two perturbation of the identity. }
   \label{fig:lin_eigen_story}
\end{figure*}

\subsection{Generalised linear model: ERM}
\label{sec:erm_logistic}

\subsubsection{Trigger alignment peaks with trigger strength}

\begin{proposition}[Trigger alignment peaks in ERM]
\label{prop:alignment_peaks}
Under Assumptions~\ref{ass:asymptotics}--\ref{ass:orthogonality} and
the mixture model~\eqref{eq:mixture}, the mean trigger alignment
$\EE[\bv^\top\tilde\btheta]$, as a function of~$\alpha$, is maximised
at some finite $\alpha^\star \in (0,\infty)$. Specifically as $n,p \rightarrow \infty$, $\EE[\bv^\top\tilde\btheta] = g(\alpha) + o_\alpha (1)$, where $g(\cdot)$ has the asymptotic (in $\alpha$) behaviour
\begin{enumerate}[label=(\roman*), itemsep=-1mm]
    \item  if $|L'(x)| \leq C_2/(x^{1+\epsilon}+1)$ for $x > 0$,
    then
$g(\alpha) =
O\!\bigl(\alpha^{-\epsilon/(2+\epsilon)}\bigr)$;
    \item if $|L'(x)| \leq C_3\,e^{-C_4 x}$ for $x > 0$
\emph{(e.g.\ logistic loss)}, then
$g(\alpha) = O(\log\alpha/\alpha)$.
\end{enumerate}
Moreover by Theorem~\ref{thm:loureiro}, $\bv^\top\hat\btheta - \EE[\bv^\top\tilde\btheta] \to 0$
in probability with $\hat{\btheta}$ defined in \eqref{eq:erm}.
\end{proposition}

\begin{proof}[Proof sketch]
From~\eqref{thm:loureiro}, $\EE[\bv^\top\tilde\btheta]$ is
controlled by $\eta_2 = \phi\EE_{\xi \sim \NN(0,1)} f(M_2 + \sigma\xi) $, which depends on the poisoned-class
margin~$M_2$.
The variance $\sigma^2$ is bounded independently of~$\alpha$ since we may write $\sigma^2 = \EE[\tilde{\btheta}^\top \C \tilde{\btheta}]$, and
the regularisation yields the clean bound $\|\hat\btheta\|^2 \leq 2L(0)/\lambda$, which we can pass onto $\tilde \btheta$.
For large $\alpha$, $M_2$ grows and pushes $L'$ into its decaying
tail in the expectation defining $\eta_2$, yielding
$\eta_2 \leq C/(|M_2|^{1+\epsilon}+1)$.
Combined with the resolvent bound
$M_2 \leq C'\eta_2\alpha^2$,
one obtains $M_2 = O(\alpha^{2/(2+\epsilon)})$ and hence
$\EE[\bv^\top\tilde\btheta] = O(\alpha^{-\epsilon/(2+\epsilon)}) \to 0$.
Since $\EE[\bv^\top\tilde\btheta] = 0$ at $\alpha = 0$ and is positive
for small $\alpha > 0$, a maximum exists at finite $\alpha^\star$.
The full proof appears in Appendix~\ref{app:proof_alignment_peaks}.
\end{proof}

\subsection{Generalised linear model: population loss}
\label{sec:pop_logistic}

\begin{proposition}[Benign convergence]
\label{prop:pop_convergence}
Under Assumption~\ref{ass:loss} with $\phi < \tfrac{1}{2}$ and
$\bv$ orthogonal to the Krylov subspace of $\C$ generated
by $\bmu$, i.e.\ $\bv^\top\C^k\bmu = 0$ for each $k \geq 0$
(cf.\ Assumption~\ref{ass:orthogonality}), 
the population minimiser $\btheta(\alpha)$
of~\eqref{eq:pop_loss} satisfies
$\btheta(\alpha) \to \btheta_{\mathrm{ben}}$ as
$\alpha \to \infty$ (defined in \eqref{eq:ben_loss}).
If additionally $\bmu$ is an eigenvector of~$\C$, then
$\bmu^\top\btheta(0) < \bmu^\top\btheta_{\mathrm{ben}}$.
\end{proposition}

\begin{proof}[Proof sketch]
\textit{Convergence to $\btheta_{\mathrm{ben}}$.}
Consider the competitor
$\hat\btheta_\alpha = \btheta_{\mathrm{ben}} + \alpha^{-1/2}\bv$.
For a poisoned sample $\tilde\x \sim \NN(\alpha\bv-\bmu,\C)$,
the margin satisfies
$\hat\btheta_\alpha^\top\tilde\x = \btheta_{\mathrm{ben}}^\top\tilde\x + \alpha^{-1/2}\bv^\top\tilde\x$,
where $\alpha^{-1/2}\bv^\top\tilde\x$ has mean
$\alpha^{1/2}\|\bv\|^2 - \alpha^{-1/2}\bmu^\top\bv \to +\infty$
and variance $\alpha^{-1}\bv^\top\C\bv \to 0$.
Hence the poisoned margin diverges and the poisoned loss
vanishes by Assumption~\ref{ass:loss}.
Since $\hat\btheta_\alpha \to \btheta_{\mathrm{ben}}$, we obtain
$\limsup_\alpha \mathcal{L}_{\mathrm{pop}}(\hat\btheta_\alpha;\alpha)
\leq \mathcal{L}_{\mathrm{ben}}(\btheta_{\mathrm{ben}})$.
The lower bound $L \geq 0$ and convexity forces
$\btheta(\alpha) \to \btheta_{\mathrm{ben}}$.

\textit{Eigenvector case.}
When $\bmu$ is an eigenvector of $\C$ with eigenvalue $\lambda_\mu$,
the gradient in any direction orthogonal to $\bmu$ is zero at any
critical point, so both $\btheta(0)$ and $\btheta_{\mathrm{ben}}$
are collinear with $\bmu$. Writing $\btheta(0) = a_0^*\bmu$ and
$\btheta_{\mathrm{ben}} = a_{\mathrm{ben}}\bmu$, optimality of
$a_{\mathrm{ben}}$ gives
\[
\frac{d}{da}\mathcal{L}_{\mathrm{pop}}(a\bmu;0) \Big |_{a = a_{\mathrm{ben}}}
= \phi\,\EE\!\bigl[L'(B_{\mathrm{ben}})
\bigl(-\|\bmu\|^2 + \sqrt{\lambda_\mu}\|\bmu\|\,\xi\bigr)\bigr],
\]
where $B_{\mathrm{ben}} = a_{\mathrm{ben}}(-\|\bmu\|^2 + \sqrt{\lambda_\mu}\|\bmu\|\xi)$
and $\xi\sim\NN(0,1)$.
Applying Stein's lemma yields
$\phi\|\bmu\|^2\!\left(-\EE[L'(B_{\mathrm{ben}})]
+ a_{\mathrm{ben}}\lambda_\mu\,\EE[L''(B_{\mathrm{ben}})]\right)>0$,
forcing $a_0^*<a_{\mathrm{ben}}$.
See Appendix~\ref{app:proof_pop_convergence}. 
\end{proof}

Under Assumption~\ref{ass:orthogonality}, Lemma~\ref{lem:cone}(ii) gives
$\bv^\top\btheta_{\mathrm{ben}} = a\,\bv^\top(\tfrac{\lambda}{1-\phi}\I+\tau\C)^{-1}\bmu = 0$,
so $\bv^\top\btheta(\alpha) \to 0$ as $\alpha \to \infty$, while
$\bv^\top\btheta(\alpha) > 0$ for moderate $\alpha$, confirming
that the trigger alignment peaks at finite~$\alpha$ in the population
regime as well.

\begin{proposition}[One-step monotonicity]
\label{prop:one_step}
Under Assumption~\ref{ass:loss} with $\phi < \tfrac{1}{2}$,
$\lambda > 0$, and $\bv^\top\bmu = 0$,
\[
\bmu^\top\nabla\mathcal{L}_{\mathrm{pop}}
(\btheta_{\mathrm{ben}};\alpha)
> 0
\quad \text{for every } \alpha \geq 0.
\]
In particular, one step of gradient descent on
$\mathcal{L}_{\mathrm{pop}}(\,\cdot\,;\alpha)$ starting from
$\btheta_{\mathrm{ben}}$ decreases $\bmu^\top\btheta(\alpha)$.
\end{proposition}

Combined with Proposition~\ref{prop:pop_convergence}, 
this gives the benign-alignment ordering: $\bmu^\top\btheta(\alpha)$ rises from $\bmu^\top\btheta(0)$ toward $\bmu^\top\btheta_{\mathrm{ben}}$.

\begin{proof}[Proof sketch]
Write $\bar{L}_\alpha(\btheta)
:= \EE_{\x\sim\NN(\alpha\bv-\bmu,\C)}[L(\x^\top\btheta)]$.
Since $\btheta_{\mathrm{ben}}$ minimises $\mathcal{L}_{\mathrm{ben}}$,
\(
\bmu^\top\nabla\mathcal{L}_{\mathrm{pop}}(\btheta_{\mathrm{ben}};\alpha)
\;=\; \phi\,\bmu^\top\nabla h_\alpha(\btheta_{\mathrm{ben}}).
\)
Setting $B := \x^\top\btheta_{\mathrm{ben}}$ for
$\x\sim\NN(\alpha\bv-\bmu,\C)$ and using $\bv^\top\bmu=0$, Stein's
lemma gives
\begin{equation}
\label{eq:stein_grad}
\bmu^\top\nabla \bar{L}_\alpha(\btheta_{\mathrm{ben}})
\;=\;
\underbrace{-\|\bmu\|^2\EE[L'(B)]}_{>\,0}
\;+\;
\underbrace{(\bmu^\top\C\btheta_{\mathrm{ben}})\EE[L''(B)]}_{\ge\,0}.
\end{equation}
The first term is strictly positive since $L'<0$.
For the second, Lemma~\ref{lem:cone}(ii) gives
$\btheta_{\mathrm{ben}} = a(\tfrac{\lambda}{1-\phi}\I+\tau\C)^{-1}\bmu$ with $a>0$,
so $\bmu^\top\C\btheta_{\mathrm{ben}} = a\,\bmu^\top\C(\tfrac{\lambda}{1-\phi}\I+\tau\C)^{-1}\bmu \geq 0$.
Hence both terms are non-negative and the first is strictly positive.
See Appendix~\ref{app:proof_one_step}.
\end{proof}

\section{Comparing ERM and information limit}
\label{sec:comparison}

The three phenomena described in Sections~\ref{sec:main_results}
and~\ref{sec:exact} arise in both the proportional regime and information limit.
In this section, we explain precisely how alignments translate into clean
accuracy and attack success rate, and why the proportional regime captures
finite-sample effects that 
do not appear in the information limit.

\paragraph{From alignment to accuracy.}
The clean test accuracy of $\btheta$ on the Gaussian mixture depends on the \emph{benign alignment} $\btheta^\top\bmu$ and the \emph{total variance} $\btheta^\top\C\btheta$. For $\Phi$ the standard normal CDF:
\begin{equation}
\label{eq:clean_acc}
  \mathrm{Acc}_{\mathrm{clean}}(\btheta)
  = \Phi\!\left(\btheta^\top\bmu/\sqrt{\btheta^\top\C\btheta}\right),
\end{equation}

For intuition we assume throughout this section that $\bmu,\bv$ are eigenvectors of $\C$ with eigenvalues $s_\mu^2,s_v^2$; the qualitative conclusions hold in general (Appendix~\ref{app:variance_decomp_general}). With $h_\mu := \EE[\bmu^\top\tilde\btheta]$ and $h_v := \EE[\bv^\top\tilde\btheta]$, the asymptotic margin variance decomposes as
\begin{equation}
\label{eq:variance_decomp}
  \sigma^2 \;:=\; \EE\!\left[\tilde\btheta^\top\C\tilde\btheta\right]
  \;=\;
  \underbrace{\tfrac{s_\mu^2}{\|\bmu\|^2}\,h_\mu^2}_{\text{mean signal}}
  \;+\;
  \underbrace{s_v^2\,h_v^2}_{\text{trigger signal}}
  \;+\;
  \underbrace{\tfrac{\gamma}{n}\,\tr\!\bigl[
  \mathbf{R}(\lambda,\tau)^2
  \C^2\bigr]}_{\text{fundamental noise }\zeta},
\end{equation}
the cross term vanishing by Assumption~\ref{ass:orthogonality}. The first two summands come from the mean of $\tilde\btheta$; $\zeta$ captures the $O(1/n)$ fluctuations in~\eqref{thm:loureiro}, scales with $\kappa = p/n$, and vanishes in the information limit. The quadratic form concentrates on $\sigma^2$, so~\eqref{eq:clean_acc} with $\btheta=\tilde\btheta$ is faithful asymptotically.

In the information limit ($p\ll n$) there is no noise floor, $\sigma^2 \approx s_\mu^2 h_\mu^2/\|\bmu\|^2 + s_v^2 h_v^2$, so for large $\alpha$, $h_v\to 0$ (Proposition~\ref{prop:alignment_peaks}), giving $\mathrm{Acc}_{\mathrm{clean}} \approx \Phi(\|\bmu\|/\sqrt{s_\mu^2})$, \emph{independent} of $h_\mu$ and hence of $\alpha$. This contradicts Figure~\ref{fig:clean_acc_cifar}, where clean accuracy is observed to rise with $\alpha$ (and hence with $h_\mu$)

\paragraph{Why clean accuracy increases with trigger strength.}
The noise floor $\zeta > 0$ breaks this degeneracy. Neglecting the (small, see Table~\ref{tab:sigma2_decomposition}) trigger term, $\mathrm{Acc}_{\mathrm{clean}} \approx g(h_\mu) := \Phi(h_\mu/\sqrt{s_\mu^2 h_\mu^2/\|\bmu\|^2 + \zeta})$, which has $g'(h_\mu)>0$ whenever $\zeta>0$. As $\alpha\to\infty$, $h_\mu$ rises to its unpoisoned value (Proposition~\ref{prop:pop_convergence}), improving clean accuracy. With $\zeta=0$ this collapses to $\|\bmu\|/\sqrt{s_\mu^2}$ and the improvement is invisible.

\paragraph{Empirical validation.}
\begin{wraptable}{r}{0.5\textwidth}
\vspace{-1.2em}
\centering
\small
\caption{Decomposition of $\sigma^2 = \EE[\tilde\btheta^\top\C\tilde\btheta]$ for empirical $\bmu$, $\C$ from CIFAR-10 (classes 0 \& 1), at $\lambda = 10^{-4}$, $\phi = 0.05$. General $\C$, not the eigenvector simplification; see Appendix~\ref{app:experimental_details} for more details. Here $\mathbf{A} := \mathbf{R}(\lambda, \tau)^2\C$}
\label{tab:sigma2_decomposition}
\begin{tabular}{lrr}
\toprule
Component & Value & \% \\
\midrule
$(\eta_1-\eta_2)^2\,\bmu^\top\mathbf{A}\bmu$    & $4.79$    & $79.93$ \\
$\eta_2^2\,\bv^\top\mathbf{A}\bv$ & $0.003$   & $0.46$  \\
$2(\eta_1-\eta_2)\eta_2\,\bmu^\top\mathbf{A}\bv$               & $-0.003$  & $-0.05$ \\
$\zeta=\tfrac{\gamma}{n}\tr[\mathbf{A}\C]$ & $1.18$    & $19.67$ \\
\bottomrule
\end{tabular}
\vspace{-1em}
\end{wraptable}
Table~\ref{tab:sigma2_decomposition} decomposes $\sigma^2$ into its 4 components, including the cross term not present in the eigenvector simplification ~\eqref{eq:variance_decomp} using the fixed-point system of Theorem~\ref{thm:loureiro}. The mean-direction signal carries $79.9\%$ of the variance; the trigger and cross terms are negligible ($<0.5\%$ combined), consistent with Assumption~\ref{ass:orthogonality}. The noise floor $\zeta$ contributes $19.67\%$, driven by $\kappa = p/n \approx 0.32$. Setting $\zeta = 0$ (the information limit) underestimates $\sigma^2$ and conceals the dependence on $h_\mu$.

\paragraph{Attack success rate.}
A test sample from the negative class poisoned with trigger $\alpha_{\mathrm{test}}\bv$ has mean $\alpha_{\mathrm{test}}\bv - \bmu$, giving attack success rate $\Phi((\alpha_{\mathrm{test}}h_v - h_\mu)/\sigma)$, with $\sigma$ including both $s_v^2 h_v^2$ and $\zeta$ via~\eqref{eq:variance_decomp}. In both regimes $h_v$ peaks at finite $\alpha^\star$ (Proposition~\ref{prop:alignment_peaks} for ERM; Proposition~\ref{prop:pop_convergence} for population risk) while $h_\mu$ rises monotonically, so the attack success rate also peaks at finite $\alpha$; in the proportional regime, $\zeta$ further moderates it---an effect absent in the population limit.


\section{Conclusion and limitations}
\label{sec:conclusion}


We analysed backdoor poisoning in regularised generalised linear models
on high-dimensional Gaussian mixtures. Closed-form squared-loss expressions
and a general fixed-point analysis yield three phenomena: (i) clean accuracy
\emph{increases} with training trigger strength, driven by a finite-sample
noise floor $\zeta \propto \kappa$ absent in the information limit;
(ii) attack success peaks at a finite trigger strength, as large triggers
push the loss into its tail; and (iii) the minimum
eigenvector of the covariance is the most effective trigger direction.
Experiments on CIFAR-10, Gaussian surrogates closely
match the theory, and ResNet-18 experiments show the same qualitative phenomena beyond the convex setting.

\paragraph{Future work.}
Our analysis rests on a Gaussian mixture data model and linear
(or generalised linear) classifiers; while experiments suggest the
phenomena persist for deep networks, extending the theory to non-convex feature maps is an interesting future direction.
Other directions include multi-class settings and non-uniform poisoning strategies.
Finally, our work characterises the attack but does not propose a
defence; developing defences informed by the identified mechanisms is
an important direction for future work. Backdoor poisoning is a known threat~\citep{gu_badnets_2019,li_backdoor_2024}, and by identifying regimes in which stronger attacks \emph{fail} our analysis is most directly useful to defenders and auditors.

\section*{Acknowledgments}
DF is funded by the Charles Coulson Scholarship. DF and JK also acknowledge support from His Majesty's Government in the development of this research. HYG and IS were partially supported by the Israel Science Foundation grant no. 777/25.

\bibliographystyle{plainnat}
\bibliography{references, references-2}

\appendix
\newpage
\addtocontents{toc}{\protect\etocdepthtag.toc{appendix}}
\etocsettocstyle{\subsection*{Appendix contents}}{}
\etocsettagdepth{mainpaper}{none}
\etocsettagdepth{appendix}{subsection}
\tableofcontents
\newpage
\input{appendix_linear.tex}
\newpage
\input{appendix_logistic.tex}
\newpage
\input{appendix_zeta.tex}
\newpage
\input{appendix_experiments.tex}
\end{document}

%% file: appendix_linear.tex
\section{Additional linear-regression results and proofs}
\label{app:linear_regression}

This appendix proves the linear-regression results stated in
Section~\ref{sec:exact}. We work directly with the proxy expectation
\[
\EE[\tilde{\btheta}],
\]
where \(\tilde{\btheta}\) is the Gaussian proxy from
Theorem~\ref{thm:loureiro}. The dependence on the trigger strength
\(\alpha\) is kept implicit. We use throughout the resolvent notation
\[
\mathbf{R}(\lambda,\tau)=(\lambda\I+\tau\C)^{-1}
\]
introduced in Theorem~\ref{thm:loureiro}.

The appendix is organized as follows. First we derive the square-loss reduction
for general covariance and obtain the exact projection formulas. We then expand
these formulas under the generic trigger-orthogonality condition
\(g_{\mu v}=o(1)\), as used in Proposition~\ref{prop:lin_peak_general} and
Corollary~\ref{cor:lin_clean_general}. Finally, we prove the eigenvector
specialization and record the isotropic formulas used in the main text.

\subsection{General-covariance reduction}

For the squared loss, the high-dimensional fixed-point equations reduce to a
finite-dimensional deterministic calculation. The following proposition records
the reduction and the projection formulas used throughout the linear-regression
analysis.

\begin{proposition}[Square-loss proxy expectation and projection formulas]
\label{prop:lin_exact_projections}
Let
\[
L(t)=\frac12(1-t)^2,
\qquad
\bmu_1=\bmu,
\qquad
\bmu_2=\alpha\bv-\bmu,
\qquad
(\pi_1,\pi_2)=(1-\phi,\phi).
\]
For the squared loss,
\[
\prox_{\delta L}(x)=\frac{x+\delta}{1+\delta},
\qquad
f(x):=-L'(\prox_{\delta L}(x))=\frac{1-x}{1+\delta}.
\]
Writing
\[
\tau:=\frac{1}{1+\delta},
\]
the scalar fixed point satisfies
\begin{equation}
\label{eq:lin_tau_eta}
\tau=\frac{1}{1+\delta},
\qquad
\delta=\frac1n\tr\!\left[\C \mathbf{R}(\lambda,\tau)\right],
\qquad
\eta_c=\pi_c\tau(1-M_c),
\qquad
M_c=\bmu_c^\top\EE[\tilde{\btheta}].
\end{equation}
In particular, the equation determining \(\tau\) does not involve the trigger
strength \(\alpha\).

Define
\begin{equation}
\label{eq:lin_mubar}
\bar{\bmu}(\alpha)
:=
\sum_{c=1}^2\pi_c\bmu_c
=
(1-2\phi)\bmu+\phi\alpha\bv
\end{equation}
and
\begin{equation}
\label{eq:lin_S}
S(\alpha)
:=
\sum_{c=1}^2\pi_c\bmu_c\bmu_c^\top
=
\bmu\bmu^\top
-\phi\alpha(\bmu\bv^\top+\bv\bmu^\top)
+\phi\alpha^2\bv\bv^\top .
\end{equation}
Then the proxy expectation is the unique solution of
\begin{equation}
\label{eq:lin_master}
\bigl(\lambda\I+\tau\C+\tau S(\alpha)\bigr)\EE[\tilde{\btheta}]
=
\tau\bar{\bmu}(\alpha).
\end{equation}

Assume further that \(\bmu\) and \(\bv\) are linearly independent. Let
\[
U:=[\,\bmu\ \ \bv\,],
\]
and define the \(2\times2\) resolvent Gram matrix
\[
G:=U^\top \mathbf{R}(\lambda,\tau) U
=
\begin{pmatrix}
g_{\mu\mu} & g_{\mu v}\\
g_{\mu v} & g_{vv}
\end{pmatrix},
\]
where
\[
g_{\mu\mu}:=\bmu^\top \mathbf{R}(\lambda,\tau)\bmu,
\qquad
g_{\mu v}:=\bmu^\top \mathbf{R}(\lambda,\tau)\bv,
\qquad
g_{vv}:=\bv^\top \mathbf{R}(\lambda,\tau)\bv .
\]
Set
\[
\Delta_G:=\det(G)=g_{\mu\mu}g_{vv}-g_{\mu v}^2.
\]
Since \(\mathbf{R}(\lambda,\tau)\succ0\) and \(U\) has full column rank, \(G\succ0\),
and hence \(\Delta_G>0\).

The projections
\[
h_\mu(\alpha):=\bmu^\top\EE[\tilde{\btheta}],
\qquad
h_v(\alpha):=\bv^\top\EE[\tilde{\btheta}]
\]
are given by
\begin{equation}
\label{eq:lin_hmu_exact}
h_\mu(\alpha)
=
\frac{
\tau\Bigl((1-2\phi)g_{\mu\mu}
+\phi\alpha g_{\mu v}
+\tau\phi(1-\phi)\alpha^2\Delta_G\Bigr)
}
{D_{\mathrm{proj}}(\alpha)}
\end{equation}
and
\begin{equation}
\label{eq:lin_hv_exact}
h_v(\alpha)
=
\frac{
\tau\Bigl((1-2\phi)g_{\mu v}
+\phi\alpha g_{vv}
+2\tau\phi(1-\phi)\alpha\Delta_G\Bigr)
}
{D_{\mathrm{proj}}(\alpha)},
\end{equation}
where
\begin{equation}
\label{eq:lin_Dh}
D_{\mathrm{proj}}(\alpha)
=
1+\tau g_{\mu\mu}
-2\tau\phi\alpha g_{\mu v}
+\tau\phi\alpha^2 g_{vv}
+\tau^2\phi(1-\phi)\alpha^2\Delta_G .
\end{equation}
\end{proposition}

\begin{proof}
The proximal formula follows by direct minimization:
\[
\prox_{\delta L}(x)
=
\argmin_{u\in\RR}
\left\{
\frac{\delta}{2}(1-u)^2+\frac12(u-x)^2
\right\}
=
\frac{x+\delta}{1+\delta}.
\]
Hence
\[
f(x):=-L'\!\bigl(\prox_{\delta L}(x)\bigr)
=
1-\prox_{\delta L}(x)
=
\frac{1-x}{1+\delta}.
\]
With \(\tau=(1+\delta)^{-1}\), this becomes
\[
f(x)=\tau(1-x),
\qquad
-f'(x)=\tau.
\]
Substituting this identity into Theorem~\ref{thm:loureiro} gives
\[
\tau=\frac{1}{1+\delta},
\qquad
\delta=\frac1n\tr\!\left[\C \mathbf{R}(\lambda,\tau)\right].
\]
Moreover, since \(r_c\sim\NN(M_c,\sigma^2)\),
\[
\eta_c
=
\pi_c\EE[f(r_c)]
=
\pi_c\tau(1-M_c),
\]
which proves~\eqref{eq:lin_tau_eta}.

Let
\[
m(\alpha):=\eta_1\bmu_1+\eta_2\bmu_2 .
\]
Taking expectations in the Gaussian proxy representation gives
\[
(\lambda\I+\tau\C)\EE[\tilde{\btheta}]=m(\alpha).
\]
Using \(\eta_c=\pi_c\tau(1-M_c)\) and
\(M_c=\bmu_c^\top\EE[\tilde{\btheta}]\), we obtain
\[
m(\alpha)
=
\tau\sum_{c=1}^2\pi_c\bmu_c
-
\tau\sum_{c=1}^2\pi_c\bmu_c\bmu_c^\top\EE[\tilde{\btheta}]
=
\tau\bar{\bmu}(\alpha)-\tau S(\alpha)\EE[\tilde{\btheta}].
\]
Therefore
\[
(\lambda\I+\tau\C+\tau S(\alpha))\EE[\tilde{\btheta}]
=
\tau\bar{\bmu}(\alpha),
\]
which proves~\eqref{eq:lin_master}. The coefficient matrix is positive
definite:
\[
\lambda\I\succ0,
\qquad
\tau\C\succeq0,
\qquad
S(\alpha)=\sum_{c=1}^2\pi_c\bmu_c\bmu_c^\top\succeq0.
\]
Thus the solution is unique. Expanding
\(\bmu_1=\bmu\), \(\bmu_2=\alpha\bv-\bmu\), and
\((\pi_1,\pi_2)=(1-\phi,\phi)\) gives
\eqref{eq:lin_mubar} and~\eqref{eq:lin_S}.

It remains to derive the projection formulas. Write
\[
d(\alpha):=
\begin{pmatrix}
1-2\phi\\
\phi\alpha
\end{pmatrix},
\qquad
K(\alpha):=
\begin{pmatrix}
1 & -\phi\alpha\\
-\phi\alpha & \phi\alpha^2
\end{pmatrix}.
\]
Then
\[
\bar{\bmu}(\alpha)=Ud(\alpha),
\qquad
S(\alpha)=UK(\alpha)U^\top .
\]
Since
\[
m(\alpha)=\eta_1\bmu_1+\eta_2\bmu_2\in\operatorname{span}\{\bmu,\bv\},
\]
there is a unique vector \(x(\alpha)\in\RR^2\) such that
\[
m(\alpha)=Ux(\alpha).
\]
Moreover,
\[
\EE[\tilde{\btheta}]=\mathbf{R}(\lambda,\tau)Ux(\alpha).
\]
The identity
\[
m(\alpha)=\tau\bar{\bmu}(\alpha)-\tau S(\alpha)\EE[\tilde{\btheta}]
\]
therefore becomes
\[
Ux
=
\tau Ud-\tau UKU^\top \mathbf{R}(\lambda,\tau)Ux.
\]
Since \(U\) has full column rank,
\[
x=\tau d-\tau K Gx,
\qquad
G:=U^\top \mathbf{R}(\lambda,\tau)U.
\]
Equivalently,
\[
(I_2+\tau K(\alpha)G)x(\alpha)=\tau d(\alpha).
\]

Now define
\[
h(\alpha):=
\begin{pmatrix}
h_\mu(\alpha)\\
h_v(\alpha)
\end{pmatrix}
=
\begin{pmatrix}
\bmu^\top\EE[\tilde{\btheta}]\\
\bv^\top\EE[\tilde{\btheta}]
\end{pmatrix}.
\]
Since \(\EE[\tilde{\btheta}]=\mathbf{R}(\lambda,\tau)Ux(\alpha)\), we have
\[
h(\alpha)=Gx(\alpha).
\]
Multiplying the reduced system by \(G\) gives
\[
(I_2+\tau G K(\alpha))h(\alpha)
=
\tau Gd(\alpha).
\]
This is a \(2\times2\) linear system. Its determinant is
\[
\det(I_2+\tau G K(\alpha))
=
1+\tau\tr(GK(\alpha))+\tau^2\det(G)\det(K(\alpha)).
\]
Because
\[
\tr(GK(\alpha))
=
g_{\mu\mu}-2\phi\alpha g_{\mu v}+\phi\alpha^2g_{vv}
\]
and
\[
\det(K(\alpha))=\phi(1-\phi)\alpha^2,
\]
the determinant is precisely \(D_{\mathrm{proj}}(\alpha)\) in
\eqref{eq:lin_Dh}.

We also note that this denominator is positive. Indeed, \(K(\alpha)\succeq0\)
and \(G\succ0\). Hence \(G K(\alpha)\) is similar to the symmetric positive
semidefinite matrix \(G^{1/2}K(\alpha)G^{1/2}\), so all eigenvalues of
\(G K(\alpha)\) are nonnegative. Since \(\tau>0\), the matrix
\(I_2+\tau G K(\alpha)\) is invertible and
\[
D_{\mathrm{proj}}(\alpha)>0 .
\]

Applying Cramer's rule to
\[
(I_2+\tau G K(\alpha))h(\alpha)
=
\tau Gd(\alpha)
\]
yields \eqref{eq:lin_hmu_exact} and~\eqref{eq:lin_hv_exact}.
\end{proof}

\subsection{Consequences under generic trigger orthogonality}

We now expand the general-covariance formulas under the condition
\[
g_{\mu v}=\bmu^\top \mathbf{R}(\lambda,\tau)\bv=o(1),
\]
which follows from Assumption~\ref{ass:orthogonality}. This is the condition
used in Proposition~\ref{prop:lin_peak_general} and
Corollary~\ref{cor:lin_clean_general}.

\begin{proof}[Proof of Proposition~\ref{prop:lin_peak_general}]
Set
\[
\varepsilon:=g_{\mu v},
\qquad
m:=g_{\mu\mu},
\qquad
q:=g_{vv},
\qquad
\Delta_\varepsilon:=mq-\varepsilon^2.
\]
By Proposition~\ref{prop:lin_exact_projections},
\[
h_v(\alpha)
=
\frac{
\tau\Bigl((1-2\phi)\varepsilon
+\phi\alpha q
+2\tau\phi(1-\phi)\alpha\Delta_\varepsilon\Bigr)
}
{
1+\tau m
-2\tau\phi\alpha\varepsilon
+\tau\phi\alpha^2q
+\tau^2\phi(1-\phi)\alpha^2\Delta_\varepsilon
}.
\]
Equivalently,
\[
h_v(\alpha)
=
\frac{n_\varepsilon+A_\varepsilon\alpha}
{B+d_\varepsilon\alpha+C_\varepsilon\alpha^2},
\]
where
\[
n_\varepsilon:=\tau(1-2\phi)\varepsilon,
\qquad
B:=1+\tau m,
\qquad
d_\varepsilon:=-2\tau\phi\varepsilon,
\]
and
\[
A_\varepsilon
:=
\tau\phi\Bigl(q+2\tau(1-\phi)\Delta_\varepsilon\Bigr),
\qquad
C_\varepsilon
:=
\tau\phi\Bigl(q+\tau(1-\phi)\Delta_\varepsilon\Bigr).
\]
Since \(\mathbf{R}(\lambda,\tau)\succ0\) and \(\bmu,\bv\) are linearly independent,
\(\Delta_\varepsilon>0\). Hence
\[
A_\varepsilon>0,
\qquad
C_\varepsilon>0.
\]

If \(\varepsilon=o(1)\), then
\[
\Delta_\varepsilon=mq+o(1),
\qquad
n_\varepsilon=o(1),
\qquad
d_\varepsilon=o(1),
\]
and
\[
A_\varepsilon
=
\tau\phi q\bigl(1+2\tau(1-\phi)m\bigr)+o(1),
\]
\[
C_\varepsilon
=
\tau\phi q\bigl(1+\tau(1-\phi)m\bigr)+o(1).
\]
Therefore, uniformly for \(\alpha\) in any fixed compact subset of
\([0,\infty)\),
\[
h_v(\alpha)
=
\frac{
\tau\phi q\bigl(1+2\tau(1-\phi)m\bigr)\alpha
}
{
1+\tau m+
\tau\phi q\bigl(1+\tau(1-\phi)m\bigr)\alpha^2
}
+o(1).
\]
This is exactly \eqref{eq:lin_hv_peak_compact}, with
\[
A_v=\tau\phi q\bigl(1+2\tau(1-\phi)m\bigr),
\qquad
C=\tau\phi q\bigl(1+\tau(1-\phi)m\bigr).
\]

It remains to prove the finite-peak claim for the exact expression.
Differentiating
\[
h_v(\alpha)
=
\frac{n_\varepsilon+A_\varepsilon\alpha}
{B+d_\varepsilon\alpha+C_\varepsilon\alpha^2}
\]
gives
\[
h_v'(\alpha)
=
\frac{
F_\varepsilon(\alpha)
}
{
\bigl(B+d_\varepsilon\alpha+C_\varepsilon\alpha^2\bigr)^2
},
\]
where
\[
F_\varepsilon(\alpha)
=
A_\varepsilon B
-
n_\varepsilon d_\varepsilon
-
2n_\varepsilon C_\varepsilon\alpha
-
A_\varepsilon C_\varepsilon\alpha^2.
\]
The quadratic \(F_\varepsilon\) has negative leading coefficient and
\[
F_\varepsilon(0)
=
A_\varepsilon B-n_\varepsilon d_\varepsilon
=
A_\varepsilon B
+
2\tau^2\phi(1-2\phi)\varepsilon^2
>0.
\]
The discriminant of \(F_\varepsilon\) is
\[
(2n_\varepsilon C_\varepsilon)^2
+
4A_\varepsilon C_\varepsilon
\bigl(A_\varepsilon B-n_\varepsilon d_\varepsilon\bigr)
>0.
\]
Thus \(F_\varepsilon\) has two real roots. Their product is
\[
\frac{
A_\varepsilon B-n_\varepsilon d_\varepsilon
}
{
-A_\varepsilon C_\varepsilon
}
<0,
\]
so exactly one root is positive. Denote this positive root by
\(\alpha_{\star,\varepsilon}\). Therefore \(h_v\) is strictly increasing on
\((0,\alpha_{\star,\varepsilon})\) and strictly decreasing on
\((\alpha_{\star,\varepsilon},\infty)\).

The positive root is
\[
\alpha_{\star,\varepsilon}
=
-\frac{n_\varepsilon}{A_\varepsilon}
+
\left[
\left(\frac{n_\varepsilon}{A_\varepsilon}\right)^2
+
\frac{
B-\frac{n_\varepsilon d_\varepsilon}{A_\varepsilon}
}
{C_\varepsilon}
\right]^{1/2}.
\]
Since \(n_\varepsilon,d_\varepsilon=o(1)\),
\(A_\varepsilon=A_v+o(1)\), and \(C_\varepsilon=C+o(1)\), we obtain
\[
\alpha_{\star,\varepsilon}^2
=
\frac{B}{C}+o(1),
\]
as claimed.
\end{proof}

\begin{proof}[Proof of Corollary~\ref{cor:lin_clean_general}]
Set again
\[
\varepsilon:=g_{\mu v},
\qquad
m:=g_{\mu\mu},
\qquad
q:=g_{vv},
\qquad
\Delta_\varepsilon:=mq-\varepsilon^2.
\]
By Proposition~\ref{prop:lin_exact_projections},
\[
h_\mu(\alpha)
=
\frac{
\tau\Bigl((1-2\phi)m
+\phi\alpha\varepsilon
+\tau\phi(1-\phi)\alpha^2\Delta_\varepsilon\Bigr)
}
{
1+\tau m
-2\tau\phi\alpha\varepsilon
+\tau\phi\alpha^2q
+\tau^2\phi(1-\phi)\alpha^2\Delta_\varepsilon
}.
\]
If \(\varepsilon=o(1)\), then, uniformly for \(\alpha\) in any fixed compact
subset of \([0,\infty)\),
\[
h_\mu(\alpha)
=
\frac{
\tau(1-2\phi)m
+
\tau^2\phi(1-\phi)mq\,\alpha^2
}
{
1+\tau m
+
\tau\phi q\bigl(1+\tau(1-\phi)m\bigr)\alpha^2
}
+o(1).
\]
This is the claimed leading-order formula with
\[
A_\mu^{(0)}=\tau(1-2\phi)m,
\qquad
A_\mu^{(2)}=\tau^2\phi(1-\phi)mq.
\]

Set
\[
B:=1+\tau m,
\qquad
C:=\tau\phi q\bigl(1+\tau(1-\phi)m\bigr).
\]

The leading-order curve has the form
\[
h_\mu^{(0)}(\alpha)
=
\frac{A_\mu^{(0)}+A_\mu^{(2)}\alpha^2}{B+C\alpha^2}.
\]
Differentiating gives
\[
\frac{d}{d\alpha}h_\mu^{(0)}(\alpha)
=
\frac{
2\alpha\bigl(A_\mu^{(2)}B-A_\mu^{(0)}C\bigr)
}
{
(B+C\alpha^2)^2
}.
\]
A direct simplification yields
\[
A_\mu^{(2)}B-A_\mu^{(0)}C
=
\tau^2\phi^2mq
\bigl(1+2\tau(1-\phi)m\bigr)
>0.
\]
Hence the leading-order curve is strictly increasing for every
\(\alpha>0\).

Finally, on every compact interval \(K\subset(0,\infty)\), the derivative
of \(h_\mu^{(0)}\) is bounded below by a positive constant. The exact derivative
is a continuous rational function of \((\varepsilon,\alpha)\), and the
denominator is bounded away from zero on \(K\). Therefore the exact derivative
converges uniformly to the derivative of \(h_\mu^{(0)}\) on \(K\). It follows
that
\[
h_\mu'(\alpha)>0
\]
eventually uniformly over \(\alpha\in K\).

\end{proof}

\subsection{Dependence on the poisoning fraction\label{sec:phi}}

We also record how the square-loss projections depend on the poisoning
fraction \(\phi\). This complements the dependence on the trigger strength
\(\alpha\): at fixed \(\alpha\), increasing the fraction of poisoned examples can
have different effects on the benign and trigger projections.

\begin{figure*}[t]
    \centering
    \begin{subfigure}[t]{0.485\textwidth}
        \centering
        \includegraphics[width=\linewidth]{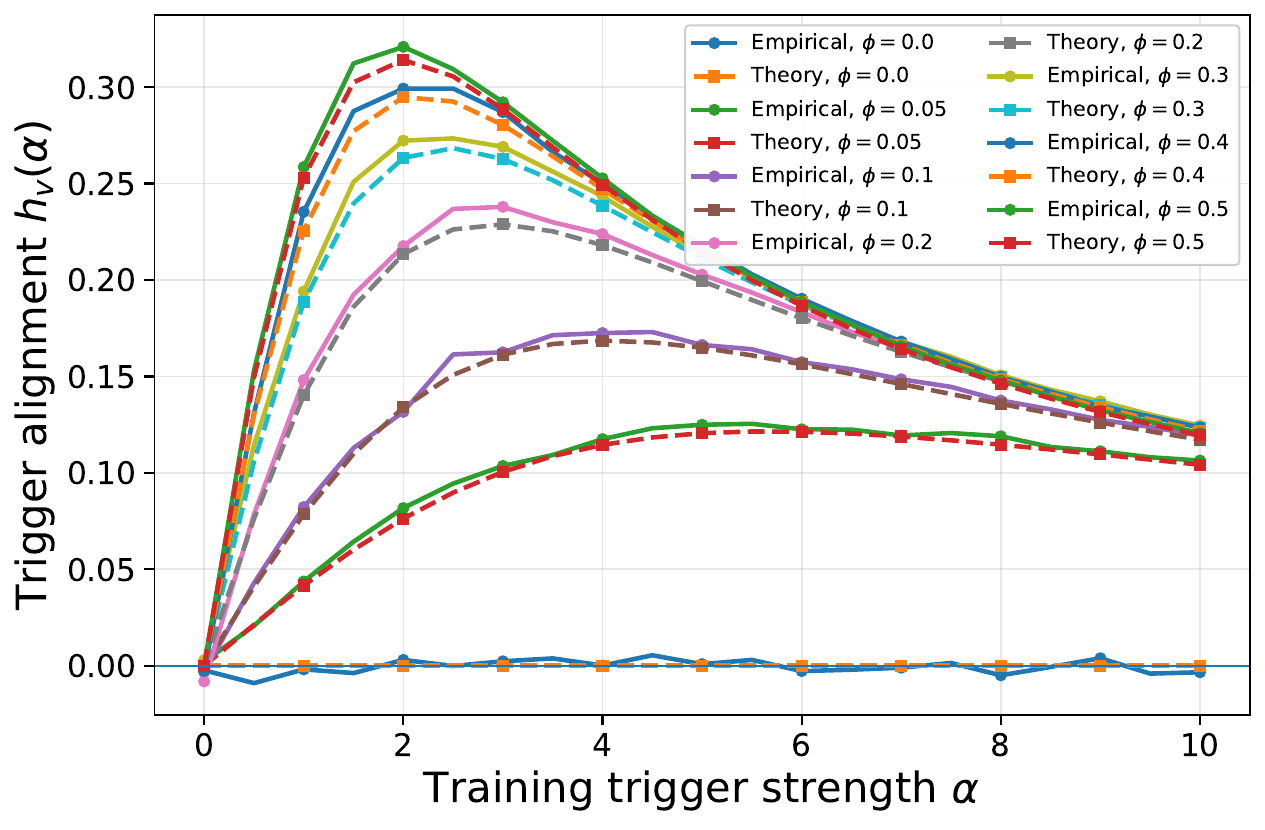}
        \caption{Trigger alignment \(h_v(\alpha)\).}
    \end{subfigure}
    \hfill
    \begin{subfigure}[t]{0.485\textwidth}
        \centering
        \includegraphics[width=\linewidth]{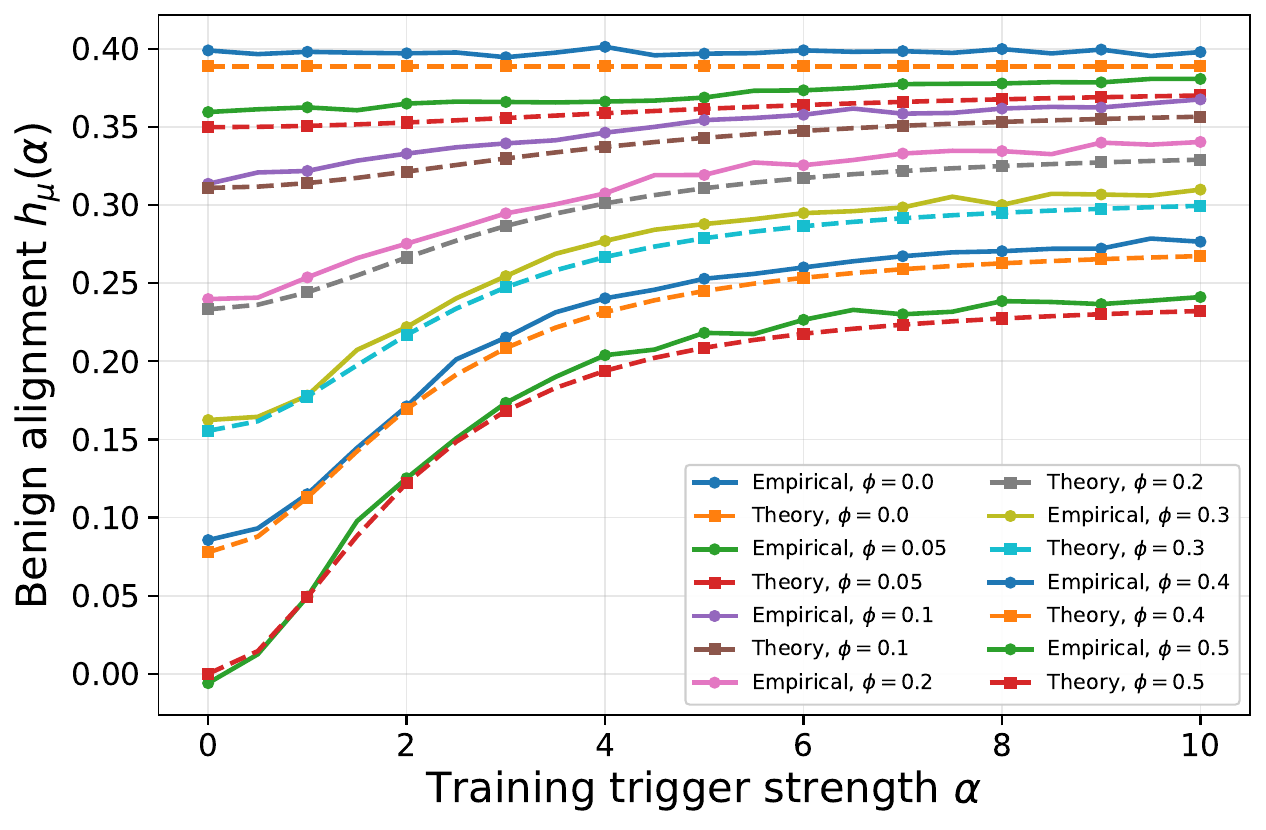}
        \caption{Benign alignment \(h_\mu(\alpha)\).}
    \end{subfigure}
    \caption{Effect of the poisoning fraction \(\phi\) in the square-loss
    model. Curves show empirical ridge estimates and the corresponding theory
    predictions as the training trigger strength \(\alpha\) varies. Increasing
    \(\phi\) can increase the trigger alignment at small \(\alpha\), but can
    decrease it once \(\alpha\) is large. The benign alignment decreases with
    \(\phi\) at fixed \(\alpha\), to leading order.}
    \label{fig:lin_phi_sweep_appendix}
\end{figure*}

Figure~\ref{fig:lin_phi_sweep_appendix} illustrates this dependence. The left
panel plots the trigger alignment \(h_v(\alpha)\) for several poisoning
fractions, and the right panel plots the benign alignment \(h_\mu(\alpha)\).
The boundary curves \(\phi=0\) and, if shown, \(\phi=1/2\), are included only
as visual references; the formal statements in the main text assume
\(0<\phi<1/2\).

To explain the behavior, write
\[
m:=g_{\mu\mu},
\qquad
q:=g_{vv},
\]
and define
\[
D_\phi(\alpha)
:=
1+\tau m+\tau\phi q\bigl(1+\tau(1-\phi)m\bigr)\alpha^2 .
\]
In the square-loss setting, \(\tau,m,q\) are fixed with respect to \(\phi\):
the scalar \(\tau\) depends only on \((\lambda,\C,\kappa)\), not on the
poisoning fraction. Under generic trigger orthogonality, the leading-order
trigger projection is
\[
h_v(\alpha,\phi)
=
\frac{
\tau\phi q\bigl(1+2\tau(1-\phi)m\bigr)\alpha
}
{
D_\phi(\alpha)
}
+o(1).
\]
Differentiating the exact rational expression and then using
\(g_{\mu v}=o(1)\) gives
\[
\frac{\partial}{\partial\phi}
h_v(\alpha,\phi)
=
\frac{
\tau q\alpha
\left[
(1+\tau m)\bigl(1+2\tau m(1-2\phi)\bigr)
-
\phi^2\tau^2mq\alpha^2
\right]
}
{
D_\phi(\alpha)^2
}
+o(1).
\]
Since \(D_\phi(\alpha)>0\), the leading-order sign is controlled by
\[
(1+\tau m)\bigl(1+2\tau m(1-2\phi)\bigr)
-
\phi^2\tau^2mq\alpha^2 .
\]
Thus the trigger projection is not monotone in \(\phi\) in general. For
sufficiently small positive \(\alpha\), the first term dominates, so increasing
\(\phi\) increases \(h_v\). For sufficiently large \(\alpha\), the negative
term
\[
-\phi^2\tau^2mq\alpha^2
\]
dominates, so increasing \(\phi\) decreases \(h_v\).

The benign projection has a simpler dependence on \(\phi\). Its leading-order
form is
\[
h_\mu(\alpha,\phi)
=
\frac{
\tau(1-2\phi)m+\tau^2\phi(1-\phi)mq\alpha^2
}
{
D_\phi(\alpha)
}
+o(1).
\]
Again, differentiating the exact rational expression and then using
\(g_{\mu v}=o(1)\) gives
\[
\frac{\partial}{\partial\phi}
h_\mu(\alpha,\phi)
=
-
\frac{
m\tau
\left[
2+2\tau m
+2\tau q\alpha^2\phi
+2\tau^2mq\alpha^2\phi^2
+\tau^2q^2\alpha^4\phi^2
\right]
}
{
D_\phi(\alpha)^2
}
+o(1).
\]
All terms inside the bracket are positive and \(D_\phi(\alpha)>0\). Hence, to
leading order,
\[
\frac{\partial}{\partial\phi}h_\mu(\alpha,\phi)<0.
\]
Thus, increasing the poisoning fraction decreases the benign projection at fixed
\(\alpha\), while its effect on the trigger projection depends on the trigger
strength.

\subsection{Eigenvector reduction}

We next prove the eigenvector specialization directly for the projections
\(h_\mu(\alpha)\) and \(h_v(\alpha)\). This avoids introducing auxiliary
coefficient notation and mirrors the statement in the main text.

\begin{proof}[Proof of Corollary~\ref{cor:lin_peak_eigen}]
Assume
\[
\bmu^\top\bv=0,
\qquad
\C\bmu=s_\mu^2\bmu,
\qquad
\C\bv=s_v^2\bv.
\]
Set
\[
L_\mu:=\lambda+\tau s_\mu^2,
\qquad
r:=\|\bmu\|^2.
\]
Then
\[
A_\mu=L_\mu+\tau r,
\qquad
B_v=\lambda+\tau s_v^2,
\]
and
\[
P_\mu=L_\mu+\tau(1-\phi)r,
\qquad
Q_\mu=L_\mu+2\tau(1-\phi)r.
\]
Since
\[
\mathbf{R}(\lambda,\tau)=(\lambda\I+\tau\C)^{-1},
\]
the eigenvector assumptions give
\[
\mathbf{R}(\lambda,\tau)\bmu=\frac{1}{L_\mu}\bmu,
\qquad
\mathbf{R}(\lambda,\tau)\bv=\frac{1}{B_v}\bv.
\]
Therefore
\[
g_{\mu\mu}
=
\bmu^\top \mathbf{R}(\lambda,\tau)\bmu
=
\frac{r}{L_\mu},
\]
\[
g_{\mu v}
=
\bmu^\top \mathbf{R}(\lambda,\tau)\bv
=
0,
\]
and, using the standing normalization \(\|\bv\|=1\),
\[
g_{vv}
=
\bv^\top \mathbf{R}(\lambda,\tau)\bv
=
\frac{1}{B_v}.
\]
Hence
\[
\Delta_G=g_{\mu\mu}g_{vv}-g_{\mu v}^2
=
\frac{r}{L_\mu B_v}.
\]

Substituting these identities into the denominator
\eqref{eq:lin_Dh} gives
\[
D_{\mathrm{proj}}(\alpha)
=
1+\tau\frac{r}{L_\mu}
+
\tau\phi\alpha^2\frac{1}{B_v}
+
\tau^2\phi(1-\phi)\alpha^2\frac{r}{L_\mu B_v}.
\]
Equivalently,
\[
D_{\mathrm{proj}}(\alpha)
=
\frac{
A_\mu B_v+\tau\phi P_\mu\alpha^2
}
{L_\mu B_v}.
\]
Thus, with
\[
D_{\mathrm{eig}}(\alpha):=A_\mu B_v+\tau\phi P_\mu\alpha^2,
\]
we have
\[
D_{\mathrm{proj}}(\alpha)
=
\frac{D_{\mathrm{eig}}(\alpha)}{L_\mu B_v}.
\]

Now use the exact projection formula \eqref{eq:lin_hmu_exact}. Since
\(g_{\mu v}=0\),
\[
h_\mu(\alpha)
=
\frac{
\tau\left((1-2\phi)\frac{r}{L_\mu}
+
\tau\phi(1-\phi)\alpha^2\frac{r}{L_\mu B_v}\right)
}
{
D_{\mathrm{proj}}(\alpha)
}.
\]
Multiplying numerator and denominator by \(L_\mu B_v\), we get
\[
h_\mu(\alpha)
=
r
\frac{
\tau(1-2\phi)B_v+\tau^2\phi(1-\phi)\alpha^2
}
{
D_{\mathrm{eig}}(\alpha)
}.
\]
This is the claimed formula for the benign projection.

Similarly, using \eqref{eq:lin_hv_exact} and \(g_{\mu v}=0\),
\[
h_v(\alpha)
=
\frac{
\tau\left(
\phi\alpha\frac{1}{B_v}
+
2\tau\phi(1-\phi)\alpha\frac{r}{L_\mu B_v}
\right)
}
{
D_{\mathrm{proj}}(\alpha)
}.
\]
Again multiplying numerator and denominator by \(L_\mu B_v\), we obtain
\[
h_v(\alpha)
=
\frac{
\tau\phi\alpha
\left(L_\mu+2\tau(1-\phi)r\right)
}
{
D_{\mathrm{eig}}(\alpha)
}
=
\frac{
\tau\phi Q_\mu\alpha
}
{
D_{\mathrm{eig}}(\alpha)
}.
\]
This proves the projection formulas in \eqref{eq:lin_ab_eigen_compact}.

It remains to prove the monotonicity and peak claims. Write
\[
D_{\mathrm{eig}}(\alpha)=D_0+D_2\alpha^2,
\qquad
D_0:=A_\mu B_v,
\qquad
D_2:=\tau\phi P_\mu .
\]
For the trigger projection,
\[
h_v(\alpha)
=
\frac{\tau\phi Q_\mu\alpha}{D_0+D_2\alpha^2}.
\]
Since \(\tau\phi Q_\mu>0\),
\[
h_v'(\alpha)
=
\frac{
\tau\phi Q_\mu(D_0-D_2\alpha^2)
}
{
(D_0+D_2\alpha^2)^2
}.
\]
Thus \(h_v\) is strictly increasing when \(\alpha^2<D_0/D_2\) and strictly
decreasing when \(\alpha^2>D_0/D_2\). Its unique maximizer satisfies
\[
(\alpha_\star^{\mathrm{eig}})^2
=
\frac{D_0}{D_2}
=
\frac{A_\mu B_v}{\tau\phi P_\mu},
\]
which is \eqref{eq:lin_alpha_star_eigen_compact}.

For the benign projection, write
\[
h_\mu(\alpha)
=
r\frac{N_0+N_2\alpha^2}{D_0+D_2\alpha^2},
\]
where
\[
N_0:=\tau(1-2\phi)B_v,
\qquad
N_2:=\tau^2\phi(1-\phi).
\]
Differentiating gives
\[
h_\mu'(\alpha)
=
r\frac{2\alpha(N_2D_0-N_0D_2)}
{(D_0+D_2\alpha^2)^2}.
\]
Using the definitions above,
\[
N_2D_0-N_0D_2
=
\tau^2\phi^2B_v
\bigl(L_\mu+2\tau(1-\phi)r\bigr)
=
\tau^2\phi^2B_vQ_\mu.
\]
This quantity is strictly positive. Therefore
\[
h_\mu'(\alpha)>0
\qquad
\text{for every } \alpha>0.
\]
This completes the proof of Corollary~\ref{cor:lin_peak_eigen}.
\end{proof}

\subsection{Isotropic specialization}

For completeness, we record the fully expanded isotropic formulas referenced
in the main text.

\begin{corollary}[Isotropic specialization]
\label{cor:lin_isotropic}
Assume
\[
\C=\I,
\qquad
\bmu^\top\bv=0,
\qquad
\|\bv\|=1.
\]
Then \(s_\mu^2=s_v^2=1\), and Corollary~\ref{cor:lin_peak_eigen} gives
\begin{align}
h_\mu(\alpha)
&=
\EE[\bmu^\top\tilde{\btheta}]
\nonumber\\
&=
\|\bmu\|^2
\frac{
\tau(1-2\phi)(\lambda+\tau)+\tau^2\phi(1-\phi)\alpha^2
}
{
(\lambda+\tau+\tau\|\bmu\|^2)(\lambda+\tau)
+
\tau\phi\bigl(\lambda+\tau+\tau(1-\phi)\|\bmu\|^2\bigr)\alpha^2
},
\label{eq:lin_hmu_iso}
\\[4pt]
h_v(\alpha)
&=
\EE[\bv^\top\tilde{\btheta}]
\nonumber\\
&=
\frac{
\tau\phi\alpha\bigl(\lambda+\tau+2\tau(1-\phi)\|\bmu\|^2\bigr)
}
{
(\lambda+\tau+\tau\|\bmu\|^2)(\lambda+\tau)
+
\tau\phi\bigl(\lambda+\tau+\tau(1-\phi)\|\bmu\|^2\bigr)\alpha^2
}.
\label{eq:lin_hv_iso}
\end{align}
The trigger projection has a unique maximizer at
\begin{equation}
\label{eq:lin_alpha_star_iso}
(\alpha^\star_{\mathrm{iso}})^2
=
\frac{(\lambda+\tau+\tau\|\bmu\|^2)(\lambda+\tau)}
{\tau\phi\bigl(\lambda+\tau+\tau(1-\phi)\|\bmu\|^2\bigr)}.
\end{equation}
In particular, \(h_v(\alpha)>0\) for every \(\alpha>0\), increases up to
\(\alpha^\star_{\mathrm{iso}}\), and decreases thereafter.
\end{corollary}

\begin{proof}
If \(\C=\I\), then
\[
s_\mu^2=s_v^2=1.
\]
Substituting these values into Corollary~\ref{cor:lin_peak_eigen} gives
\[
A_\mu=\lambda+\tau+\tau\|\bmu\|^2,
\qquad
B_v=\lambda+\tau,
\]
and
\[
P_\mu=\lambda+\tau+\tau(1-\phi)\|\bmu\|^2,
\qquad
Q_\mu=\lambda+\tau+2\tau(1-\phi)\|\bmu\|^2.
\]
The projection formulas \eqref{eq:lin_hmu_iso} and~\eqref{eq:lin_hv_iso}
follow by substituting these identities into
\eqref{eq:lin_ab_eigen_compact}. The peak location
\eqref{eq:lin_alpha_star_iso} follows from
\eqref{eq:lin_alpha_star_eigen_compact}.
\end{proof}

%% file: appendix_logistic.tex
\section{Proofs for non-quadratic losses }
\label{app:proofs}

Throughout this appendix, we write
$\mathbf{R} = \mathbf{R}(\lambda,\tau) = (\lambda\I+\tau\C)^{-1}$ for the resolvent.

\begin{remark}[Conventions for the proofs]
\label{rem:conventions}
Assumption~\ref{ass:asymptotics} is really a statement about a
sequence $\{(\bmu_n, \C_n, \bv_n)\}$ in dimension $p_n$ with
$p_n/n\to\kappa$; we suppress the index throughout and write
``$=o(1)$'' for terms vanishing along this sequence as $n\to\infty$
with $\alpha$ fixed.  All constants in the bounds below depend only
on $L$, $\lambda$, $\phi$, $\sup_n\|\C_n\|$ and $\sup_n\|\bmu_n\|$.

\end{remark}

\subsection{Proof of Proposition~\ref{prop:alignment_peaks}}
\label{app:proof_alignment_peaks}

\begin{proof}
We use the notation of Theorem~\ref{thm:loureiro}.
The proof proceeds in six steps: we first bound the variance parameter
$\sigma^2$ independently of the trigger strength~$\alpha$, then
decompose the alignment equations under the orthogonality assumption,
bound the key quantity $\eta_2$ as a function of the poisoned-class
margin~$M_2$, and finally combine these bounds to establish the
decay rate.

\medskip
\noindent\textbf{Step 1: Bounding $\sigma^2$ independently of
$\alpha$.}
We first show that the variance parameter $\sigma^2$ from the
fixed-point system remains bounded as $\alpha$ varies.
Since $\hat\btheta$ minimises the regularised empirical loss
$\mathcal{L}_n$, comparing with the zero vector gives
\[
\mathcal{L}_n(\hat\btheta) \leq \mathcal{L}_n(\mathbf{0}) = L(0).
\]
Expanding the left-hand side and using non-negativity of the
individual loss terms, we obtain
\[
\frac{\lambda}{2}\|\hat\btheta\|^2
\leq \mathcal{L}_n(\hat\btheta) \leq L(0),
\]
so $\|\hat\btheta\|^2 \leq 2L(0)/\lambda$.
Applying Theorem~\ref{thm:loureiro} to the pseudo-Lipschitz function
$\btheta\mapsto\|\btheta\|^2$ gives
$\|\hat\btheta\|^2 \rightarrow \EE[\|\tilde\btheta\|^2]$ in probability, and since
$\EE[\|\tilde\btheta\|^2]$ is deterministic, this forces
$\EE[\|\tilde\btheta\|^2] \leq 2L(0)/\lambda + o_\alpha(1)$
as $n\to\infty$, and therefore
\[
\sigma^2
= \EE[\tilde\btheta^\top\C\tilde\btheta]
\leq \EE[\|\tilde\btheta\|^2]\,\|\C\|
\leq \sigma_{\max}^2 + o_\alpha(1),
\qquad
\sigma_{\max}^2 := \frac{2L(0)\|\C\|}{\lambda},
\]
where the $o_\alpha(1)$ is uniform in~$\alpha$ in the sense that the
limiting bound $\sigma_{\max}^2$ does not depend on~$\alpha$.
This bound will be used in Step~3 to control the Gaussian
expectations defining~$\eta_2$; since $\sigma_{\max}^2$ is finite, the
$o_\alpha(1)$ correction is absorbed into the constants of the bounds
that follow without affecting their structural form, and re-emerges
explicitly in the final rate.

\medskip
\noindent\textbf{Step 2: Alignment decomposition under
orthogonality.}
Recall from Theorem~\ref{thm:loureiro} that the Gaussian proxy
$\tilde\btheta$ has mean
\[
\EE[\tilde\btheta]
= \mathbf{R}\bigl[\eta_1\bmu + \eta_2(\alpha\bv - \bmu)\bigr]
= \mathbf{R}\bigl[(\eta_1 - \eta_2)\bmu + \eta_2\alpha\bv\bigr],
\]
where the two classes have means $\bmu_1 = \bmu$ and
$\bmu_2 = \alpha\bv - \bmu$ with probabilities $1-\phi$ and $\phi$
respectively.
The fixed-point means Theorem~\ref{thm:loureiro} are therefore
\begin{align*}
M_1
&= \bmu^\top\mathbf{R}\bigl[(\eta_1-\eta_2)\bmu
   + \eta_2\alpha\bv\bigr]
 = (\eta_1-\eta_2)\bmu^\top\mathbf{R}\bmu
   + \eta_2\alpha\,\bv^\top\mathbf{R}\bmu, \\[4pt]
M_2
&= (\alpha\bv-\bmu)^\top\mathbf{R}\bigl[(\eta_1-\eta_2)\bmu
   + \eta_2\alpha\bv\bigr].
\end{align*}
Under Assumption~\ref{ass:orthogonality}, the cross term
$\bv^\top\mathbf{R}\bmu = o_{\alpha}(1)$ is negligible in the limit.
Applying this to the first equation gives
$M_1 = (\eta_1-\eta_2)\bmu^\top\mathbf{R}\bmu + o_{\alpha}(1)$.
For the second, we expand:
\begin{align*}
M_2
&= \alpha(\eta_1-\eta_2)
   \underbrace{\bv^\top\mathbf{R}\bmu}_{o_{\alpha}(1)}
 - (\eta_1-\eta_2)\bmu^\top\mathbf{R}\bmu
 + \eta_2\alpha^2\bv^\top\mathbf{R}\bv
 - \eta_2\alpha
   \underbrace{\bmu^\top\mathbf{R}\bv}_{o_{\alpha}(1)} \\
&= \eta_2\alpha^2\bv^\top\mathbf{R}\bv
 - (\eta_1-\eta_2)\bmu^\top\mathbf{R}\bmu + o_{\alpha}(1).
\end{align*}
We therefore obtain the simplified system
\begin{align}
M_1 &= (\eta_1 - \eta_2)\,\bmu^\top\mathbf{R}\bmu + o_{\alpha}(1),
\label{eq:M1_app} \\
M_2 &= \eta_2\alpha^2\,\bv^\top\mathbf{R}\bv - M_1 + o_{\alpha}(1).
\label{eq:M2_app}
\end{align}
The trigger alignment is computed similarly:
\[
\EE[\tilde\btheta^\top\bv]
= \bv^\top\mathbf{R}
  \bigl[(\eta_1-\eta_2)\bmu + \eta_2\alpha\bv\bigr]
= (\eta_1-\eta_2)
  \underbrace{\bv^\top\mathbf{R}\bmu}_{o_{\alpha}(1)}
  + \eta_2\alpha\,\bv^\top\mathbf{R}\bv
= \eta_2\alpha\,\bv^\top\mathbf{R}\bv + o_{\alpha}(1).
\]
Noting that
$M_1 + M_2 = \eta_2\alpha^2\,\bv^\top\mathbf{R}\bv$
from~\eqref{eq:M1_app}--\eqref{eq:M2_app}, we also have the useful
identity
\begin{equation}
\label{eq:alignment_identity_app}
\EE[\tilde\btheta^\top\bv]
= \frac{M_1 + M_2}{\alpha} + o_{\alpha}(1).
\end{equation}
This alignment vanishes at $\alpha = 0$ (since $\eta_2 = 0$ when
there is no poisoned class) and is positive for $\alpha > 0$
(since $\eta_2 > 0$ and $\bv^\top\mathbf{R}\bv > 0$).

We also note that
$|M_1| = |\bmu^\top\EE[\tilde\btheta]|
\leq \|\bmu\|\,\sqrt{\EE[\|\tilde\btheta\|^2]}
\leq M_1^{\max} + o_\alpha(1)$,
where $M_1^{\max} := \|\bmu\|\sqrt{2L(0)/\lambda}$, by Jensen's
inequality and Step~1.
Since the first term of~\eqref{eq:M2_app} is non-negative,
$M_2$ is bounded below:
\[
M_2
= \eta_2\alpha^2\bv^\top\mathbf{R}\bv - M_1
\geq -|M_1|
\geq -M_1^{\max} - o_\alpha(1).
\]

\begin{remark}[Orthogonality with $n$-dependent $\tau$]
\label{rem:krylov}
Assumption~\ref{ass:orthogonality} is stated for each fixed
$\tau \geq 0$, but in Theorem~\ref{thm:loureiro} the quantity
$\tau = \tau_n$ depends on $n$.
Since $\tau_n$ is bounded, any subsequence has a convergent
sub-subsequence $\tau_{n_k} \to \tau_\infty$, and
$\bv^\top(\lambda\I+\tau_\infty\C)^{-1}\bmu = 0$ by the assumption.
A resolvent-identity argument extends this to
$\bv^\top(\lambda\I+\tau_n\C)^{-1}\bmu \to 0$ along the full
sequence.
\end{remark}

\medskip
\noindent\textbf{Step 3: Bounding $\eta_2$ as a function of $M_2$.}

We view $\eta_2$ as a function of $M_2$.
Recall from the fixed-point system that
\[
\eta_2 = \phi\,\EE_{\xi \sim \NN(0,1)}[f(M_2 + \sigma\xi)],
\]
where $f$ is defined via the proximal operator as
$f(x) = -L'\bigl(\prox_{\delta L}(x)\bigr)$.
The function $f$ satisfies the functional equation
\[
f(x) = -L'(x + \delta f(x)).
\]
Since $L$ is convex and decreasing, the proximal shift satisfies
$\delta f(x) \geq 0$, and since $L'$ is increasing (by convexity)
and strictly negative (by Assumption~\ref{ass:loss}), we obtain the
pointwise bound
\[
f(x) = -L'(x + \delta f(x)) \leq -L'(x).
\]
This allows us to replace $f$ with the simpler quantity $-L'$ in upper
bounds, at the cost of an inequality.
Using this together with the tail bounds on $L'$ from
Assumption~\ref{ass:loss}, we establish:

\begin{lemma}
\label{lem:eta2_bound}
There exists $C > 0$ independent of $\alpha$ and $n$ such that
\[
\eta_2 \;\leq\; \frac{C + o_\alpha(1)}{|M_2|^{1+\epsilon}+1},
\]
where the $o_\alpha(1)$ tracks the finite-$n$ deviation of $\sigma$
from its limiting upper bound $\sigma_{\max}$.
\end{lemma}

\begin{proof}
The bound is trivial for $M_2 \leq 0$ since $M_2$ is bounded below
(Step~2).
Hence assume $M_2 > 0$.
Using the bound $f(x) \leq -L'(x)$ established above:
\begin{align*}
\eta_2
&= \phi\,\EE_{\xi \sim \NN(0,1)} f(M_2 + \sigma \xi) \\
&\leq \phi\,\EE_{\xi}[-L'(M_2 + \sigma \xi)] \\
&= \phi\,\EE_{\xi}\!\left[
   -L'(M_2 + \sigma \xi)\,\mathbf{1}_{M_2 + \sigma\xi \geq 0}
\right]
 + \phi\,\EE_{\xi}\!\left[
   -L'(M_2 + \sigma \xi)\,\mathbf{1}_{M_2 + \sigma\xi < 0}
\right] \\
&\leq \phi\,\EE_{\xi}\!\left[
   \frac{C_2}{(M_2 + \sigma\xi)^{1+\epsilon}+1}\,
   \mathbf{1}_{M_2 + \sigma\xi \geq 0}
\right]
 + \phi\,\EE_{\xi}\!\left[
   e^{C_1(|M_2 + \sigma\xi|+1)}\,
   \mathbf{1}_{M_2 + \sigma\xi < 0}
\right],
\end{align*}
where we have applied the positive and negative tail bounds from
Assumption~\ref{ass:loss}.
Since $\sigma \leq \sigma_{\max} + o_\alpha(1)$ (Step~1), the second
term is of order $O(e^{-C_1 M_2})$ uniformly in $n$ for $n$
sufficiently large, as the Gaussian integral over
$\{\sigma\xi < -M_2\}$ is bounded.

The first term can be further partitioned, using the bound
$\sigma \leq \sigma_{\max} + o_\alpha(1)$ from Step~1:
\begin{align*}
&\phi\,\EE_{\xi}\!\left[
   \frac{C_2}{(M_2 + \sigma\xi)^{1+\epsilon}+1}\,
   \mathbf{1}_{M_2 + \sigma\xi \geq 0}
\right] \\
&\quad\leq
\phi\,\EE_{\xi}\!\left[
   \frac{C_2}{(M_2 + \sigma\xi)^{1+\epsilon}+1}\,
   \mathbf{1}_{\xi \geq -\frac{M_2}{2\sigma_{\max}}}
\right]
 + \phi\,\EE_{\xi}\!\left[
   \frac{C_2}{(M_2 + \sigma\xi)^{1+\epsilon}+1}\,
   \mathbf{1}_{\xi < -\frac{M_2}{2\sigma_{\max}}}
\right].
\end{align*}
For the second term: the integrand is bounded (since the denominator
is at least $1$), while the Gaussian measure of
$\{\xi < -M_2/(2\sigma_{\max})\}$ is $O(e^{-C'M_2^2})$ by a standard
tail bound, so this contribution is negligible.

For the first term: when $\xi \geq -M_2/(2\sigma_{\max})$ we have
$M_2 + \sigma\xi \geq M_2 - \sigma \cdot M_2/(2\sigma_{\max})
\geq M_2/2 - o_\alpha(1)\,M_2/(2\sigma_{\max})$
(since $\sigma \leq \sigma_{\max} + o_\alpha(1)$), so for $n$
sufficiently large $M_2 + \sigma\xi \geq M_2/4$, and therefore
\[
\frac{C_2}{(M_2 + \sigma\xi)^{1+\epsilon}+1}
\leq \frac{C_2}{(M_2/4)^{1+\epsilon}+1}
\leq \frac{C}{|M_2|^{1+\epsilon}+1}.
\]
The constant $C$ depends continuously on $\sigma_{\max}+o_\alpha(1)$,
giving the lemma's bound
$\eta_2 \leq (C+o_\alpha(1))/(|M_2|^{1+\epsilon}+1)$. 
\end{proof}

\medskip
\noindent\textbf{Step 4: Polynomial-tail decay rate.}
We now combine the bound on $\eta_2$ from
Lemma~\ref{lem:eta2_bound} with the alignment equations to determine
the growth rate of $M_2$ as $\alpha \to \infty$.

From~\eqref{eq:M2_app}, we have
$M_2 = \eta_2\alpha^2\bv^\top\mathbf{R}\bv - M_1$, and since $M_1$
is bounded we obtain
\[
M_2 \leq \eta_2\alpha^2\bv^\top\mathbf{R}\bv + |M_1|.
\]
Since $\|\bv\| = 1$ and
$\mathbf{R} = (\lambda\I + \tau\C)^{-1} \preceq \lambda^{-1}\I$,
we have the resolvent bound
\[
\bv^\top\mathbf{R}\bv
\leq \frac{\|\bv\|^2}{\lambda} = \frac{1}{\lambda}.
\]
Substituting the bound from Lemma~\ref{lem:eta2_bound},
$\eta_2 \leq (C+o_\alpha(1))/(|M_2|^{1+\epsilon}+1)$:
\[
M_2
\leq \frac{(C+o_\alpha(1))\,\alpha^2}{\lambda(|M_2|^{1+\epsilon}+1)} + |M_1|.
\]
Since $|M_1| \leq M_1^{\max} + o_\alpha(1)$ is bounded independently
of $\alpha$ (Step~2), for $\alpha$ sufficiently large the first term
dominates, and we may absorb $|M_1|$ into the constant.
We therefore have the self-consistent inequality
\[
M_2 \leq \frac{(C'+o_\alpha(1))\,\alpha^2}{|M_2|^{1+\epsilon}+1}
\]
for some constant $C'$ not depending on $\alpha$ or $n$. And thus for $M_2 > 0$
rearranging gives
$M_2^{2+\epsilon} \leq (C'+o_\alpha(1))\alpha^2$, i.e.\
\[
|M_2| \leq C''\,\alpha^{2/(2+\epsilon)}\,(1+o_\alpha(1)).
\]

Where we note now that this inequality holds for $M_2 < 0$ and $\alpha$ large too trivially, since $M_2$ is bounded below.

Now, recalling the identity~\eqref{eq:alignment_identity_app},
$\EE[\bv^\top\tilde\btheta] = (M_1 + M_2)/\alpha$, and using
$|M_1| \leq M_1^{\max} + o_\alpha(1)$ from Step~2 together with the
$M_2$ bound above:
\[
\EE[\tilde\btheta^\top\bv]
= \frac{M_1 + M_2}{\alpha}
\leq \frac{M_1^{\max}}{\alpha}
+ \frac{C''\,\alpha^{2/(2+\epsilon)}\,(1+o_\alpha(1))}{\alpha}
+ \frac{o_\alpha(1)}{\alpha}
= g(\alpha) + o_\alpha(1),
\]
where, using
$\frac{2}{2+\epsilon} - 1
= \frac{-\epsilon}{2+\epsilon}$
and absorbing
$M_1^{\max}/\alpha \leq M_1^{\max}\alpha^{-\epsilon/(2+\epsilon)}$
(valid for $\alpha \geq 1$) into the leading term, while the
multiplicative $o_\alpha(1)\,\alpha^{-\epsilon/(2+\epsilon)}$ is at
most $o_\alpha(1)$ for $\alpha \geq 1$,
\[
g(\alpha) \;\leq\; C'\,\alpha^{-\epsilon/(2+\epsilon)}.
\]
Since $\epsilon > 0$, $g(\alpha) \to 0$ as $\alpha \to \infty$.

\medskip
\noindent\textbf{Step 5: Concentration and existence of maximum.}
By Theorem~\ref{thm:loureiro} applied to the $1$-Lipschitz function
$\btheta\mapsto\bv^\top\btheta$, we have
$\bv^\top\hat\btheta\to_p\EE[\bv^\top\tilde\btheta]$ in probability,
so the decay bound from Step~4 on $\EE[\bv^\top\tilde\btheta]$
transfers to $\bv^\top\hat\btheta$ in probability.

Since $\EE \bv^\top\tilde\btheta = o_{\alpha}(1)$ at $\alpha = 0$
(no poisoning) and $\EE[\bv^\top \tilde\btheta] = \eta_2 \alpha^2 \bv^\top \mathbf{R} \bv$ is always positive for $\alpha > 0$ (since $\eta_2 = \EE[f(\cdot)]$ and $f(x) \geq 0$ under our assumptions on $L$), and vanishes as
$\alpha \to \infty$ (by the decay bound above), continuity
guarantees that the trigger alignment $\bv^\top\tilde\btheta$ is
maximised at some finite $\alpha^\star \in (0,\infty)$.

\medskip
\noindent\textbf{Step 6: Exponential-tail case.}
Suppose now that the loss derivative satisfies the stronger bound
$|L'(x)| \leq C_3 e^{-C_4 x}$ for $x > 0$ (this holds, for
instance, for the logistic loss $L(x) = \log(1+e^{-x})$, where
$L'(x) = -e^{-x}/(1+e^{-x})$).

We show that $\eta_2 \leq Ce^{-C'M_2}$ for $M_2$ large, by a
partition argument analogous to that in Lemma~\ref{lem:eta2_bound}.
Starting from
$\eta_2 \leq \phi\,\EE_\xi[-L'(M_2 + \sigma\xi)]$,
we split the expectation into the same regions:
\[
\eta_2
\leq \phi\,\EE_{\xi}\!\left[
   -L'(M_2+\sigma\xi)\,\mathbf{1}_{M_2+\sigma\xi \geq 0}
\right]
 + \phi\,\EE_{\xi}\!\left[
   -L'(M_2+\sigma\xi)\,\mathbf{1}_{M_2+\sigma\xi < 0}
\right].
\]

\emph{Negative region} $\{M_2 + \sigma\xi < 0\}$:
as before, the exponential bound
$|L'(x)| \leq e^{C_1(|x|+1)}$ for $x < 0$ combined with
$\sigma \leq \sigma_{\max} + o_\alpha(1)$ from Step~1 gives a
contribution of $O(e^{-C_1 M_2})$, uniformly in $n$ for $n$ large.

\emph{Positive region} $\{M_2 + \sigma\xi \geq 0\}$:
we partition further at $\xi = -M_2/(2\sigma_{\max})$, using
$\sigma \leq \sigma_{\max} + o_\alpha(1)$.
\begin{align*}
&\phi\,\EE_{\xi}\!\left[
   -L'(M_2+\sigma\xi)\,\mathbf{1}_{M_2+\sigma\xi \geq 0}
\right] \\
&\quad\leq
\phi\,\EE_{\xi}\!\left[
   C_3 e^{-C_4(M_2+\sigma\xi)}\,
   \mathbf{1}_{-\frac{M_2}{\sigma_{\max}} < \xi
   < -\frac{M_2}{2\sigma_{\max}}}
\right]
 + \phi\,\EE_{\xi}\!\left[
   C_3 e^{-C_4(M_2+\sigma\xi)}\,
   \mathbf{1}_{\xi \geq -\frac{M_2}{2\sigma_{\max}}}
\right].
\end{align*}
For the first term: the integrand is bounded, while the Gaussian
measure of
$\{\xi < -M_2/(2\sigma_{\max})\}$ is $O(e^{-C'M_2^2})$, which is
negligible compared to $e^{-C'M_2}$ for large enough $M_2$.

For the second term: when $\xi \geq -M_2/(2\sigma_{\max})$ we have
$M_2 + \sigma\xi \geq M_2/2 - o_\alpha(1)\,M_2/(2\sigma_{\max})$
(since $\sigma \leq \sigma_{\max}+o_\alpha(1)$), hence
$M_2 + \sigma\xi \geq M_2/4$ for $n$ sufficiently large. The
exponential tail bound then gives
$|L'(M_2+\sigma\xi)| \leq C_3 e^{-C_4 M_2/4}$, and since the Gaussian
integral over this region is at most $1$, the contribution is
$O(e^{-C_4 M_2/4})$.

Combining all regions yields
\[
\eta_2 \leq (C+o_\alpha(1))\,e^{-C'M_2}
\]
for constants $C, C' > 0$ not depending on $\alpha$ or $n$, with the
$o_\alpha(1)$ tracking the dependence of $C$ on $\sigma_{\max}+o_\alpha(1)$
rather than $\sigma_{\max}$ exactly.

Substituting into the alignment equation as in Step~4:
\[
M_2
\leq (C+o_\alpha(1))\,e^{-C'M_2}\frac{\alpha^2}{\lambda} + |M_1|.
\]
For $\alpha$ large, absorbing the bounded $|M_1|$ term:
\[
M_2 \leq (C''+o_\alpha(1))\,e^{-C'M_2}\alpha^2,
\]
and hence
\[
M_2\,e^{C'M_2} \leq (C''+o_\alpha(1))\,\alpha^2.
\]
Now for $M_2 > 0$ then taking logarithms of both sides:
\[
\log M_2 + C'M_2 \leq 2\log\alpha + \log(C''+o_\alpha(1)).
\]
For $M_2$ sufficiently large we have $\log M_2 \geq 0$, and therefore
\[
C'M_2 \leq 2\log\alpha + O(1) + o_\alpha(1),
\]
giving $M_2 \leq C'''\,\log\alpha + o_\alpha(1)$ as
$\alpha \to \infty$. Where again now the statement is trivially true for negative $M_2$ and large enough alpha, since $M_2$ is bounded below.
Hence,
 \[
|M_2| \leq C'''\,\log\alpha + o_\alpha(1)
\]

Finally, using the identity~\eqref{eq:alignment_identity_app} and
$|M_1| \leq M_1^{\max} + o_\alpha(1)$ from Step~2 together with the
$M_2$ bound above:
\[
\EE[\tilde\btheta^\top\bv]
= \frac{M_1 + M_2}{\alpha}
\leq \frac{M_1^{\max} + C'''\,\log\alpha}{\alpha}
+ \frac{o_\alpha(1)}{\alpha}
= g(\alpha) + o_\alpha(1),
\qquad
g(\alpha) \leq C'\,\frac{\log\alpha}{\alpha},
\]
where $M_1^{\max}/\alpha$ is absorbed into the leading
$\log\alpha/\alpha$ term for $\alpha \geq e$, and $o_\alpha(1)/\alpha
\leq o_\alpha(1)$ for $\alpha \geq 1$.
The concentration and existence-of-maximum arguments from Step~5
apply identically.
\end{proof}

\subsection{Cone Lemma}
\label{app:proof_cone}

The following structural lemma underpins much of the population-risk
analysis. It shows that the unpoisoned ($\alpha = 0$) population
minimiser, and as a special case the benign minimiser, lies in a
two-dimensional submanifold determined by the resolvent of $\C$.\footnote{This can be compared with the form of $\hat \btheta$ for the ERM with $\alpha = 0$, for which its expectation will lie on a similar submanifold.}

\begin{lemma}[Cone Lemma]
\label{lem:cone}
Suppose Assumption~\ref{ass:loss} holds with $\phi < \tfrac{1}{2}$ and
$\lambda > 0$. Then:
\begin{enumerate}[label=(\roman*),leftmargin=*]
\item The minimiser $\btheta(0)$
of~\eqref{eq:pop_loss} at $\alpha = 0$ satisfies
\[
\btheta(0) \;=\; a\,(\lambda\I+\tau\C)^{-1}\bmu
\]
for some $a>0$ and $\tau\ge 0$.
\item The benign minimiser $\btheta_{\mathrm{ben}}$
of~\eqref{eq:ben_loss} satisfies
\[
\btheta_{\mathrm{ben}} \;=\; a\,\bigl(\tfrac{\lambda}{1-\phi}\I+\tau\C\bigr)^{-1}\bmu
\]
for some $a>0$ and $\tau\ge 0$.
\end{enumerate}
Part~(ii) is the special case of part~(i) with $\phi = 0$ and $\lambda$
replaced by $\lambda/(1-\phi)$, since
$\btheta_{\mathrm{ben}}$ is the minimiser of
$\mathcal{L}_{\mathrm{ben}}/(1-\phi)
= \EE_{\x\sim\NN(\bmu,\C)}[L(\btheta^\top\x)]
+ \tfrac{\lambda/(1-\phi)}{2}\|\btheta\|^2$,
which coincides with $\mathcal{L}_{\mathrm{pop}}(\,\cdot\,;0)$ at
$\phi = 0$ and regulariser $\lambda/(1-\phi)$.
\end{lemma}

\begin{proof}
We prove (i); part (ii) follows from (i) by the rescaling noted in the
statement (and with $\phi = 0$ the condition $a > 0$ is immediate since
$L' < 0$ everywhere by Assumption~\ref{ass:loss}).

Setting
$\nabla\mathcal{L}_{\mathrm{pop}}(\btheta(0);0)
= \mathbf{0}$ and applying Stein's lemma to the Gaussian
expectations yields
$(\lambda\I+\tau\C)\btheta(0) = a\bmu$,
where, writing $\mathbf{z} \sim \NN(\mathbf{0},\I)$,
\[
A = \btheta(0)^\top\bmu
+ \btheta(0)^\top\C^{1/2}\mathbf{z},
\qquad
B = -\btheta(0)^\top\bmu
+ \btheta(0)^\top\C^{1/2}\mathbf{z},
\]
and
\[
\tau = (1-\phi)\,\EE[L''(A)] + \phi\,\EE[L''(B)] \geq 0,
\qquad
a = -(1-\phi)\,\EE[L'(A)] + \phi\,\EE[L'(B)].
\]
Here $A$ represents the margin of a class~$+1$ sample and
$B$ that of a class~$-1$ sample.
Since $\lambda > 0$ and $\tau \geq 0$, the matrix
$\lambda\I+\tau\C \succ 0$, so
$\btheta(0) = a(\lambda\I+\tau\C)^{-1}\bmu$
and $\mathrm{sign}(\bmu^\top\btheta(0))
= \mathrm{sign}(a)$.

It remains to show $a > 0$.
Suppose for contradiction that
$\bmu^\top\btheta(0) \leq 0$.
Then $A$ has nonpositive mean
$\EE[A] = \bmu^\top\btheta(0) \leq 0$ and $B$ has
nonnegative mean
$\EE[B] = -\bmu^\top\btheta(0) \geq 0$.
Since $L'$ is increasing (by convexity of~$L$) and strictly negative
(by Assumption~\ref{ass:loss}: $L$ is strictly decreasing),
a stochastic-ordering argument gives
\[
\EE[L'(A)] \leq \EE[L'(B)] < 0.
\]
Therefore
\begin{align*}
a
&= -(1-\phi)\EE[L'(A)] + \phi\EE[L'(B)] \\
&\geq -(1-\phi)\EE[L'(B)] + \phi\EE[L'(B)] \\
&= -(1-2\phi)\EE[L'(B)] > 0,
\end{align*}
using $\phi < \tfrac{1}{2}$ and $\EE[L'(B)] < 0$.
This implies $\bmu^\top\btheta(0) > 0$, contradicting
our assumption, so $a > 0$.
\end{proof}

\begin{remark}
The representation reflects that the loss depends on~$\btheta$ only
through $\btheta^\top\bmu$ and $\btheta^\top\C\btheta$, while the
regulariser selects the unique minimiser on this manifold.
For part~(i), the condition $\phi < \tfrac{1}{2}$ enforces positive
alignment $\bmu^\top\btheta(0) > 0$: the majority (positive) class
dominates the gradient. For part~(ii) the analogous statement is trivial
since there is only one class.
\end{remark}

\subsection{Proof of Proposition~\ref{prop:pop_convergence}}
\label{app:proof_pop_convergence}

\paragraph{Orthogonality assumption used in this section.}
Throughout this subsection we work with the strong form of
orthogonality, $\bv^\top\C^k\bmu = 0$ for every $k \geq 0$, rather
than the asymptotic $o_{\alpha}(1)$ form of
Assumption~\ref{ass:orthogonality}.
The reason is that the population minimiser $\btheta(\alpha)$ is
defined for any fixed dimension $p$ (including small $p$), and there
is no $n \to \infty$ limit available along which to absorb $o_{\alpha}(1)$
terms.
The same conclusions hold under the asymptotic version of
Assumption~\ref{ass:orthogonality} if one further takes
$p \to \infty$ (after $n \to \infty$): then $\bv^\top\C^k\bmu = o_{\alpha}(1)$
along the sequence, and the resolvent identity
$\bv^\top(\lambda\I+\tau\C)^{-1}\bmu = o_{\alpha}(1)$ follows with $0$ replaced by
$o_{\alpha}(1)$ at each occurrence.

\begin{proof}
The proof has two parts.
Part~(i) shows that the population minimiser converges to the benign
minimiser by constructing a competitor whose poisoned loss vanishes
as $\alpha \to \infty$ and then using strong convexity.
Part~(ii) shows that when $\bmu$ is an eigenvector of~$\C$, both
minimisers are collinear with~$\bmu$, and a derivative comparison
at $\btheta_{\mathrm{ben}}$ establishes the strict inequality.

Existence and uniqueness of $\btheta(\alpha)$ and
$\btheta_{\mathrm{ben}}$ follow from $\lambda$-strong convexity.

\medskip
\noindent\textbf{Part (i): Convergence as $\alpha \to \infty$.}

We begin with an auxiliary result whose proof is given at the end
of Part~(i).

\begin{lemma}[Gaussian expectation vanishing]
\label{lem:gauss_vanish}
Let $(X_n)_{n \geq 1}$ be Gaussian random variables with
$\EE[X_n] \to \infty$ and
$\bar\sigma^2 := \sup_{n \geq 1}\operatorname{Var}(X_n) < \infty$.
Under Assumption~\ref{ass:loss},
$\lim_{n \to \infty}\EE[L(X_n)] = 0$.
\end{lemma}

\begin{remark}\label{rem:gauss_int}
Integrating the derivative bound in Assumption~\ref{ass:loss} gives
$L(x) \leq A\,e^{C_1|x|}$ for $x < 0$ and $L(x) \leq L(0)$
for $x \geq 0$.
In particular, $\EE[L(a + b\,\xi)] < \infty$ for any finite
$a,b\in\RR$ and $\xi \sim \NN(0,1)$, since the moment-generating
function of $|\xi|$ is finite everywhere. 
\end{remark}

We next record a continuity result for the benign expectation.

\begin{lemma}[Continuity of benign expectation]
\label{lem:benign_cont}
If $\btheta_\alpha \to \btheta$ in $\ell_2$, then
$\EE_{\x\sim\NN(\bmu,\C)}[L(\x^\top\btheta_\alpha)]
\to \EE_{\x\sim\NN(\bmu,\C)}[L(\x^\top\btheta)]$.
\end{lemma}
\begin{proof}
Write $X_\alpha := \x^\top\btheta_\alpha$ and
$X := \x^\top\btheta$.
Since $\x\sim\NN(\bmu,\C)$, both are Gaussian:
$X_\alpha\sim\NN(\mu_\alpha,\sigma_\alpha^2)$ with
$\mu_\alpha := \bmu^\top\btheta_\alpha$ and
$\sigma_\alpha^2 := \btheta_\alpha^\top\C\btheta_\alpha$,
and $X\sim\NN(\mu,\sigma^2)$ with
$\mu := \bmu^\top\btheta$ and
$\sigma^2 := \btheta^\top\C\btheta$.
Since $\btheta_\alpha\to\btheta$ in $\ell_2$,
we have $\mu_\alpha\to\mu$ and $\sigma_\alpha\to\sigma$.

\emph{Almost sure convergence.}
Since $\btheta_\alpha\to\btheta$ deterministically,
for each realisation of $\x$ with $\|\x\|<\infty$
(which holds a.s.\ for Gaussian $\x$),
Cauchy--Schwarz gives
\[
|X_\alpha - X|
= |\x^\top(\btheta_\alpha - \btheta)|
\leq \|\x\|\,\|\btheta_\alpha - \btheta\| \to 0.
\]
Continuity of $L$ then yields
$L(X_\alpha)\to L(X)$ almost surely.

\emph{Upper bound via coupling.}
Fix $\epsilon>0$ and set
$\bar\sigma_\epsilon^2 := \sigma^2+\epsilon$.
Since $\sigma_\alpha\to\sigma$, for all sufficiently large
$\alpha$ we have $\sigma_\alpha^2\leq\bar\sigma_\epsilon^2$.
By Gaussianity we can couple an independent
$\xi\sim\NN(0,1)$ with
$\tilde\sigma_\alpha
:= (\bar\sigma_\epsilon^2-\sigma_\alpha^2)^{1/2}\geq 0$
so that
\[
Y_\alpha := X_\alpha + \tilde\sigma_\alpha\,\xi
\sim\NN\!\bigl(\mu_\alpha,\,\bar\sigma_\epsilon^2\bigr).
\]
By the conditional Jensen inequality (using convexity of $L$),
\[
\EE[L(Y_\alpha)]
= \EE\bigl[\EE[L(X_\alpha+\tilde\sigma_\alpha\xi)\mid X_\alpha]\bigr]
\geq \EE\bigl[L\bigl(
  \EE[X_\alpha+\tilde\sigma_\alpha\xi\mid X_\alpha]\bigr)\bigr]
= \EE[L(X_\alpha)],
\]
whence
$\EE[L(X_\alpha)]
\leq \EE[L(Y_\alpha)]
= \EE[L(\mu_\alpha+\bar\sigma_\epsilon\,\xi)]$.

We apply dominated convergence to the right-hand side.
Pointwise, $L(\mu_\alpha+\bar\sigma_\epsilon\,\xi)
\to L(\mu+\bar\sigma_\epsilon\,\xi)$ by continuity of $L$.
Since $L$ is decreasing, for all $\alpha$ large enough that
$\mu_\alpha\geq\mu-1$ we have the dominator
\[
L(\mu_\alpha+\bar\sigma_\epsilon\,\xi)
\leq L\bigl((\mu-1)+\bar\sigma_\epsilon\,\xi\bigr).
\]
The right-hand side is integrable by
Remark~\ref{rem:gauss_int}.
Dominated convergence therefore yields
\[
\limsup_{\alpha\to\infty}\EE[L(X_\alpha)]
\leq
\lim_{\alpha\to\infty}\EE[L(\mu_\alpha+\bar\sigma_\epsilon\,\xi)]
= \EE[L(\mu+\bar\sigma_\epsilon\,\xi)].
\]

\emph{Conclusion.}
Since $L\geq 0$ and $L(X_\alpha)\to L(X)$ a.s., Fatou's lemma
gives
$\EE[L(X)]\leq\liminf_\alpha\EE[L(X_\alpha)]$.
The coupling bound above holds for every $\epsilon>0$.
As $\epsilon\downarrow 0$, $\bar\sigma_\epsilon\downarrow\sigma$
and $L(\mu+\bar\sigma_\epsilon\,\xi)\to L(\mu+\sigma\,\xi)$
pointwise; a further application of dominated convergence
(with integrable dominator $L(\mu-(\sigma{+}1)|\xi|)$,
cf.\ Remark~\ref{rem:gauss_int}) gives
$\EE[L(\mu+\bar\sigma_\epsilon\,\xi)]
\to\EE[L(\mu+\sigma\,\xi)]=\EE[L(X)]$.
Hence $\limsup_\alpha\EE[L(X_\alpha)]\leq\EE[L(X)]$,
and combining with Fatou's lower bound completes the proof.
\end{proof}

We now construct a sequence of competitors whose population loss
converges to the benign-only optimum.

\begin{lemma}[Competitor upper bound]
\label{lem:competitor}
Define $\hat\btheta_\alpha := \btheta_{\mathrm{ben}}
+ \alpha^{-1/2}\bv$. Then
\[
\limsup_{\alpha\to\infty}
\mathcal{L}_{\mathrm{pop}}(\hat\btheta_\alpha;\alpha)
\leq \mathcal{L}_{\mathrm{ben}}(\btheta_{\mathrm{ben}}).
\]
\end{lemma}
\begin{proof}
By Lemma~\ref{lem:cone}(ii),
$\btheta_{\mathrm{ben}} = a(\tfrac{\lambda}{1-\phi}\I+\tau\C)^{-1}\bmu$,
so Krylov orthogonality gives
$\bv^\top\btheta_{\mathrm{ben}} = 0$.
We analyse each term of the population loss separately.

\emph{Regulariser.}
Since $\bv^\top\btheta_{\mathrm{ben}} = 0 $,
\[
\tfrac{\lambda}{2}\|\hat\btheta_\alpha\|^2
- \tfrac{\lambda}{2}\|\btheta_{\mathrm{ben}}\|^2
= \tfrac{\lambda}{2}\alpha^{-1}\|\bv\|^2 \to 0.
\]

\emph{Benign loss.}
The benign loss converges by Lemma~\ref{lem:benign_cont}
since $\hat\btheta_\alpha \to \btheta_{\mathrm{ben}}$ in $\ell_2$.

\emph{Poisoned loss.}
For the poisoned term, let
$\tilde\x \sim \NN(\alpha\bv-\bmu,\C)$.
The margin $\hat\btheta_\alpha^\top\tilde\x$ is Gaussian with
mean
\[
(\alpha\bv-\bmu)^\top\hat\btheta_\alpha
= \alpha^{1/2}\|\bv\|^2
  - \bmu^\top\btheta_{\mathrm{ben}}
  - \alpha^{-1/2}\bmu^\top\bv
\to +\infty
\]
(using $\bv^\top\btheta_{\mathrm{ben}}=0$)
and variance
$\hat\btheta_\alpha^\top\C\hat\btheta_\alpha
\leq \|\C\|\,\|\hat\btheta_\alpha\|^2$,
which is bounded since
$\hat\btheta_\alpha\to\btheta_{\mathrm{ben}}$.
Lemma~\ref{lem:gauss_vanish} therefore gives
$\EE[L(\hat\btheta_\alpha^\top\tilde\x)] \to 0$.
\end{proof}

Since $L \geq 0$, the poisoned loss term is non-negative, so
$\mathcal{L}_{\mathrm{ben}}(\btheta)
\leq \mathcal{L}_{\mathrm{pop}}(\btheta;\alpha)$ pointwise for
every~$\btheta$.
Taking the infimum:
\[
\mathcal{L}_{\mathrm{ben}}(\btheta_{\mathrm{ben}})
\leq \inf_{\btheta}\mathcal{L}_{\mathrm{pop}}(\btheta;\alpha).
\]
Lemma~\ref{lem:competitor} gives the matching $\limsup$, hence
$\inf_{\btheta}\mathcal{L}_{\mathrm{pop}}(\btheta;\alpha)
\to \mathcal{L}_{\mathrm{ben}}(\btheta_{\mathrm{ben}})$.

\begin{lemma}[Convergence of minimisers]
\label{lem:min_conv}
$\btheta(\alpha) \to \btheta_{\mathrm{ben}}$ as $\alpha \to \infty$ in $\ell_2$.
\end{lemma}
\begin{proof}
Using
$\mathcal{L}_{\mathrm{ben}} \leq
\mathcal{L}_{\mathrm{pop}}(\cdot;\alpha)$ pointwise:
\[
0 \leq \mathcal{L}_{\mathrm{ben}}(\btheta(\alpha))
- \mathcal{L}_{\mathrm{ben}}(\btheta_{\mathrm{ben}})
\leq \inf_{\btheta}\mathcal{L}_{\mathrm{pop}}(\btheta;\alpha)
- \mathcal{L}_{\mathrm{ben}}(\btheta_{\mathrm{ben}}) \to 0.
\]
By $\lambda$-strong convexity of $\mathcal{L}_{\mathrm{ben}}$,
the left side is at least
$\tfrac{\lambda}{2}\|\btheta(\alpha)-\btheta_{\mathrm{ben}}\|^2$,
and therefore
$\|\btheta(\alpha) - \btheta_{\mathrm{ben}}\| \to 0$.
\end{proof}

Part~(i) follows from Lemma~\ref{lem:min_conv}.

We now give the deferred proof.

\begin{proof}[Proof of Lemma~\ref{lem:gauss_vanish}]
Write $\mu_n = \EE[X_n]$ and $\sigma_n^2 = \operatorname{Var}(X_n)$.
Let $Y_n \sim \NN(\mu_n,\bar\sigma^2)$.
Since $\sigma_n^2 \leq \bar\sigma^2$,
we can couple an independent $\xi\sim\NN(0,1)$
with $\tilde\sigma_n
:= (\bar\sigma^2-\sigma_n^2)^{1/2}\geq 0$
so that
$Y_n = X_n + \tilde\sigma_n\,\xi$.
By the conditional Jensen inequality (using convexity of $L$),
\[
\EE[L(Y_n)]
= \EE\bigl[\EE[L(X_n+\tilde\sigma_n\xi)\mid X_n]\bigr]
\geq \EE\bigl[L\bigl(
  \EE[X_n+\tilde\sigma_n\xi\mid X_n]\bigr)\bigr]
= \EE[L(X_n)],
\]
so $0 \leq \EE[L(X_n)]
\leq \EE[L(Y_n)]
= \EE[L(\mu_n+\bar\sigma\,\xi)]$.

Since $L$ is decreasing and $\mu_n > 0$ for $n$ large,
$L(\mu_n+\bar\sigma\,\xi) \leq L(\bar\sigma\,\xi)$.
This dominator is integrable by Remark~\ref{rem:gauss_int}.
Moreover $L(\mu_n+\bar\sigma\,\xi) \to 0$ pointwise,
since $\mu_n+\bar\sigma\,\xi \to +\infty$ and
$L(x)\to 0$ as $x\to+\infty$
(which follows from $L$ being non-negative, convex,
and strictly decreasing).
By dominated convergence,
$\EE[L(\mu_n+\bar\sigma\,\xi)] \to 0$,
and the squeeze
$0 \leq \EE[L(X_n)] \leq \EE[L(\mu_n+\bar\sigma\,\xi)]$
completes the proof.
\end{proof}

\medskip
\noindent\textbf{Part (ii): Eigenvector case.}
Suppose $\bmu$ is an eigenvector of $\C$ with eigenvalue~$\lambda_\mu$.
We first show that both minimisers are collinear with~$\bmu$, and then
compare their scalar coefficients.

Writing $\btheta = a\bmu + \mathbf{w}$ with $\mathbf{w} \perp \bmu$,
the variance decouples:
$\btheta^\top\C\btheta
= a^2\lambda_\mu\|\bmu\|^2 + \mathbf{w}^\top\C\mathbf{w}$.
Setting $\nabla_{\mathbf{w}}\mathcal{L}_{\mathrm{pop}}(\btheta;0)
= \mathbf{0}$ and using the fact that the gradient with respect to
$\mathbf{w}$ at any critical point satisfies
\[
\nabla_{\mathbf{w}}\mathcal{L}_{\mathrm{pop}}(\btheta;0)
= \bigl[(1-\phi)\EE[L''(A)] + \phi\EE[L''(B)]\bigr]\C\mathbf{w}
+ \lambda\mathbf{w} = \mathbf{0},
\]
forces $\mathbf{w} = \mathbf{0}$ since $\lambda > 0$
(the coefficient of $\C\mathbf{w}$ is non-negative by convexity of
$L$, so $\lambda\mathbf{w} = \mathbf{0}$ is the only solution).
The same argument applies to $\mathcal{L}_{\mathrm{ben}}$.
Both minimisers are therefore collinear with $\bmu$:
$\btheta^* = a^*\bmu$,
and it suffices to compare the scalar minimisers
\[
a_0^* = \arg\min_a\,\mathcal{L}_{\mathrm{pop}}(a\bmu;\,0),
\qquad
a_{\mathrm{ben}} = \arg\min_a\,\mathcal{L}_{\mathrm{ben}}(a\bmu).
\]
Since $\mathcal{L}_{\mathrm{pop}}(\,\cdot\,;0)$ is strictly convex
in $a$ with unique minimiser $a_0^*$, it suffices to show
\[
\frac{d}{da}\mathcal{L}_{\mathrm{pop}}(a_{\mathrm{ben}}\bmu;\,0) > 0,
\]
which forces $a_0^* < a_{\mathrm{ben}}$ by strict convexity.

By optimality of $a_{\mathrm{ben}}$,
$\frac{d}{da}\mathcal{L}_{\mathrm{ben}}(a_{\mathrm{ben}}\bmu) = 0$,
so the derivative of the population loss at $a_{\mathrm{ben}}$
reduces to the contribution from the poisoned class alone:
\[
\frac{d}{da}\mathcal{L}_{\mathrm{pop}}(a\bmu;\,0)\bigg|_{a=a_{\mathrm{ben}}}
= \phi\,\frac{d}{da}\,\EE\!\left[L\!\left(
  -a\|\bmu\|^2 + a\sqrt{\lambda_\mu}\|\bmu\|\,\xi
\right)\right]\bigg|_{a=a_{\mathrm{ben}}},
\]
where $\xi \sim \NN(0,1)$.
Differentiating inside the expectation:
\[
\frac{d}{da}\mathcal{L}_{\mathrm{pop}}(a\bmu;\,0)\bigg|_{a=a_{\mathrm{ben}}}
= \phi\,\EE\!\left[L'(B_{\mathrm{ben}})\left(-\|\bmu\|^2
+ \sqrt{\lambda_\mu}\|\bmu\|\,\xi\right)\right],
\]
where $B_{\mathrm{ben}} = a_{\mathrm{ben}}(-\|\bmu\|^2
+ \sqrt{\lambda_\mu}\|\bmu\|\xi)$ is the margin of a poisoned sample
evaluated at $\btheta_{\mathrm{ben}}$.

We now separate the deterministic and Gaussian contributions.
Applying Stein's lemma to the $\xi$ term, using
$\partial B_{\mathrm{ben}}/\partial\xi
= a_{\mathrm{ben}}\sqrt{\lambda_\mu}\|\bmu\|$:
\[
\EE[L'(B_{\mathrm{ben}})\xi]
= a_{\mathrm{ben}}\sqrt{\lambda_\mu}\|\bmu\|\,
\EE[L''(B_{\mathrm{ben}})].
\]
Substituting:
\[
\frac{d}{da}\mathcal{L}_{\mathrm{pop}}(a\bmu;\,0)\bigg|_{a=a_{\mathrm{ben}}}
= \phi\|\bmu\|^2\!\left(
\underbrace{-\EE[L'(B_{\mathrm{ben}})]}_{>0}
+ \underbrace{a_{\mathrm{ben}}\lambda_\mu\,
\EE[L''(B_{\mathrm{ben}})]}_{\geq 0}
\right) > 0,
\]
since $L'(x) < 0$ everywhere (Assumption~\ref{ass:loss}),
$a_{\mathrm{ben}} > 0$ (by the Cone Lemma), and
$\lambda_\mu, \EE[L''(B_{\mathrm{ben}})] \geq 0$
(by non-negativity of the eigenvalue and convexity of $L$).
Therefore $a_0^* < a_{\mathrm{ben}}$, i.e.\
$\bmu^\top\btheta(0) < \bmu^\top\btheta_{\mathrm{ben}}$.
\end{proof}

\subsection{Proof of Proposition~\ref{prop:one_step}}
\label{app:proof_one_step}

\begin{proof}
The idea is to decompose the population gradient at
$\btheta_{\mathrm{ben}}$ into the benign part (which vanishes by
optimality of $\btheta_{\mathrm{ben}}$) and the poisoned part, then
show the latter has a strictly positive projection onto~$\bmu$ using
Stein's lemma and the Cone Lemma.

Write
$\mathcal{L}_{\mathrm{pop}}(\btheta;\alpha)
= \mathcal{L}_{\mathrm{ben}}(\btheta)
+ \phi\,h_\alpha(\btheta)$,
where
$h_\alpha(\btheta)
= \EE_{\x\sim\NN(\alpha\bv-\bmu,\C)}[L(\x^\top\btheta)]$.
Since $\nabla\mathcal{L}_{\mathrm{ben}}(\btheta_{\mathrm{ben}})
= \mathbf{0}$ by optimality,
\begin{equation}
\label{eq:grad_decomp_app}
\bmu^\top\nabla\mathcal{L}_{\mathrm{pop}}
(\btheta_{\mathrm{ben}};\alpha)
= \phi\,\bmu^\top\nabla h_\alpha(\btheta_{\mathrm{ben}}).
\end{equation}
It suffices to show
$\bmu^\top\nabla h_\alpha(\btheta_{\mathrm{ben}}) > 0$.

Let $\x = \alpha\bv - \bmu + \C^{1/2}\mathbf{z}$ with
$\mathbf{z} \sim \NN(\mathbf{0},\I)$, and set
$B = \x^\top\btheta_{\mathrm{ben}}$.
The gradient of the poisoned expectation in the direction $\bmu$ is
\[
\bmu^\top\nabla h_\alpha(\btheta_{\mathrm{ben}})
= \EE[L'(B)\,\bmu^\top\x].
\]
Using $\bv^\top\bmu = 0$:
$\bmu^\top\x = -\|\bmu\|^2 + \bmu^\top\C^{1/2}\mathbf{z}$.
We now separate the deterministic part (involving $-\|\bmu\|^2$) from
the Gaussian part (involving $\mathbf{z}$).
Applying Stein's lemma to the Gaussian part:
\[
\EE[L'(B)\,\bmu^\top\C^{1/2}\mathbf{z}]
= (\bmu^\top\C\btheta_{\mathrm{ben}})\,\EE[L''(B)],
\]
where we used $\partial B/\partial \mathbf{z}
= \C^{1/2}\btheta_{\mathrm{ben}}$.
Therefore
\begin{equation}
\label{eq:stein_split_app}
\bmu^\top\nabla h_\alpha(\btheta_{\mathrm{ben}})
= \underbrace{-\|\bmu\|^2\EE[L'(B)]}_{>0}
+ \underbrace{(\bmu^\top\C\btheta_{\mathrm{ben}})
\EE[L''(B)]}_{\geq 0}.
\end{equation}
The first term is strictly positive since $L$ is strictly decreasing
($L' < 0$ everywhere by Assumption~\ref{ass:loss}).

For the second term, we invoke Lemma~\ref{lem:cone}(ii), which gives
$\btheta_{\mathrm{ben}} = a(\tfrac{\lambda}{1-\phi}\I+\tau\C)^{-1}\bmu$ with
$a > 0$.
Writing $\mathbf{R} = \tfrac{\lambda}{1-\phi}\I+\tau\C$:
\[
\bmu^\top\C\btheta_{\mathrm{ben}}
= a\,\bmu^\top\C\mathbf{R}^{-1}\bmu \geq 0,
\]
since $\C$ and $\mathbf{R}^{-1}$ are commuting positive semidefinite
matrices (they share the eigenbasis of $\C$) and hence their product
$\C\mathbf{R}^{-1}$ is also positive semidefinite.
Moreover $\EE[L''(B)] \geq 0$ by convexity of $L$.

Hence both terms in~\eqref{eq:stein_split_app} are non-negative and
the first is strictly positive,
giving $\bmu^\top\nabla h_\alpha(\btheta_{\mathrm{ben}}) > 0$.
Substituting into~\eqref{eq:grad_decomp_app} yields
$\bmu^\top\nabla\mathcal{L}_{\mathrm{pop}}
(\btheta_{\mathrm{ben}};\alpha) > 0$.
\end{proof}


%% file: appendix_zeta.tex
\section{Comparing ERM and information limit}

\subsection{Precise relation between ERM and information limit}
\label{app:erm_pop_bound}

\paragraph{Fixed-dimensional convergence of the empirical optimiser.}
We briefly justify the relationship between the empirical optimisation problem
\[
\widehat{\btheta}_{n}
\in
\argmin_{\btheta\in\RR^p}
\mathcal L_n(\btheta),
\qquad
\mathcal L_n(\btheta)
=
\frac1n\sum_{i=1}^n L(y_i\x_i^\top\btheta)
+
\frac{\lambda}{2}\|\btheta\|^2,
\]
and its population analogue
\[
\btheta_{\rm pop}(\alpha)
\in
\argmin_{\btheta\in\RR^p}
\mathcal L_{\rm pop}(\btheta;\alpha).
\]
Writing \(\mathbf{z}_i=y_i\x_i\), the empirical objective is
\[
\mathcal L_n(\btheta)
=
\frac1n\sum_{i=1}^n L(\btheta^\top \mathbf{z}_i)
+
\frac{\lambda}{2}\|\btheta\|^2.
\]
The population objective is the corresponding expectation under the mixture law
\[
\mathbf{z} \sim P_\alpha
:=
(1-\phi)\NN(\bmu,\C)
+
\phi\NN(\alpha\bv-\bmu,\C),
\]
namely
\[
\mathcal L_{\rm pop}(\btheta;\alpha)
=
\EE_{\mathbf{z}\sim P_\alpha}
\left[
L(\btheta^\top \mathbf{z})
\right]
+
\frac{\lambda}{2}\|\btheta\|^2.
\]

The standard empirical-risk decomposition gives
\[
\begin{aligned}
\mathcal L_{\rm pop}(\widehat{\btheta}_n;\alpha)
-
\mathcal L_{\rm pop}(\btheta_{\rm pop};\alpha)
&=
\left[
\mathcal L_{\rm pop}(\widehat{\btheta}_n;\alpha)
-
\mathcal L_n(\widehat{\btheta}_n)
\right]
\\
&\quad+
\left[
\mathcal L_n(\widehat{\btheta}_n)
-
\mathcal L_n(\btheta_{\rm pop})
\right]
\\
&\quad+
\left[
\mathcal L_n(\btheta_{\rm pop})
-
\mathcal L_{\rm pop}(\btheta_{\rm pop};\alpha)
\right].
\end{aligned}
\]
The middle term is non-positive by optimality of \(\widehat{\btheta}_n\). Hence, for any set
\(\Theta\) containing both \(\widehat{\btheta}_n\) and \(\btheta_{\rm pop}\),
\[
\mathcal L_{\rm pop}(\widehat{\btheta}_n;\alpha)
-
\mathcal L_{\rm pop}(\btheta_{\rm pop};\alpha)
\leq
2
\sup_{\btheta\in\Theta}
\left|
\mathcal L_n(\btheta)
-
\mathcal L_{\rm pop}(\btheta;\alpha)
\right|.
\]
Since the regularisation term is deterministic and appears in both objectives, the supremum reduces to
\[
\sup_{\btheta\in\Theta}
\left|
\frac1n\sum_{i=1}^n L(\btheta^\top \mathbf{z}_i)
-
\EE_{\mathbf{z}\sim P_\alpha}
L(\btheta^\top \mathbf{z})
\right|.
\]

It remains to control the uniform deviation of the function class
\[
\mathcal G_\Theta
=
\left\{
\mathbf{z}\mapsto L(\btheta^\top \mathbf{z})
:
\btheta\in\Theta
\right\}.
\]
A bounded parameter set is needed in order for this class to have finite complexity. In the present regularised problem this is natural. For example, if \(L\geq 0\), then
\[
\mathcal L_n(\widehat{\btheta}_n)
\leq
\mathcal L_n(0)
=
L(0),
\]
and therefore
\[
\frac{\lambda}{2}\|\widehat{\btheta}_n\|^2
\leq
L(0).
\]
Thus
\[
\|\widehat{\btheta}_n\|
\leq
B_\lambda
:=
\sqrt{\frac{2L(0)}{\lambda}}.
\]
The same argument applies to \(\btheta_{\rm pop}\). Hence it suffices to take
\[
\Theta_\lambda
=
\left\{
\btheta\in\RR^p:
\|\btheta\|\leq B_\lambda
\right\}.
\]

Let
\[
\mathfrak R_n(\mathcal G_{\Theta_\lambda})
=
\EE_{\mathbf{z},\sigma}
\left[
\sup_{\|\btheta\|\leq B_\lambda}
\frac1n
\sum_{i=1}^n
\sigma_i
L(\btheta^\top \mathbf{z}_i)
\right],
\]
where \(\sigma_1,\dots,\sigma_n\) are independent Rademacher variables. By the usual symmetrisation and contraction inequalities for Rademacher complexity, if \(L\) is \(G\)-Lipschitz, then
\[
\EE
\sup_{\|\btheta\|\leq B_\lambda}
\left|
\frac1n\sum_{i=1}^n L(\btheta^\top \mathbf{z}_i)
-
\EE L(\btheta^\top \mathbf{z})
\right|
\lesssim
G\,
\mathfrak R_n
\left(
\left\{
\mathbf{z}\mapsto \btheta^\top \mathbf{z}:
\|\btheta\|\leq B_\lambda
\right\}
\right).
\]
For the linear class,
\[
\begin{aligned}
\mathfrak R_n
\left(
\left\{
\mathbf{z}\mapsto \btheta^\top \mathbf{z}:
\|\btheta\|\leq B_\lambda
\right\}
\right)
&=
\EE_{\mathbf{z},\sigma}
\left[
\sup_{\|\btheta\|\leq B_\lambda}
\frac1n
\sum_{i=1}^n
\sigma_i
\btheta^\top \mathbf{z}_i
\right]
\\
&=
B_\lambda
\EE_{\mathbf{z},\sigma}
\left\|
\frac1n
\sum_{i=1}^n
\sigma_i \mathbf{z}_i
\right\|
\\
&\leq
B_\lambda
\sqrt{
\frac{\EE\|\mathbf{z}\|^2}{n}
}.
\end{aligned}
\]
For the Gaussian mixture \(P_\alpha\),
\[
\EE\|\mathbf{z}\|^2
=
\tr(\C)
+
(1-\phi)\|\bmu\|^2
+
\phi\|\alpha\bv-\bmu\|^2.
\]
Consequently,
\[
\EE
\sup_{\|\btheta\|\leq B_\lambda}
\left|
\mathcal L_n(\btheta)
-
\mathcal L_{\rm pop}(\btheta;\alpha)
\right|
\lesssim
G B_\lambda
\sqrt{
\frac{
\tr(\C)
+
(1-\phi)\|\bmu\|^2
+
\phi\|\alpha\bv-\bmu\|^2
}{n}
}.
\]
In the fixed-dimensional regime, or more generally when
\[
\EE\|\mathbf{z}\|^2 = O(p),
\]
this is the standard
\[
O\left(\sqrt{\frac{p}{n}}\right)
\]
Rademacher-complexity rate for bounded linear predictors; see, for example,
standard treatments of Rademacher complexity and linear prediction
\citep{mohriFoundationsMachineLearning2018}.

It follows that, for fixed \(p\) and fixed \(\lambda>0\),
\[
\sup_{\|\btheta\|\leq B_\lambda}
\left|
\mathcal L_n(\btheta)
-
\mathcal L_{\rm pop}(\btheta;\alpha)
\right|
\to 0
\]
in probability, and therefore
\[
\mathcal L_{\rm pop}(\widehat{\btheta}_n;\alpha)
-
\mathcal L_{\rm pop}(\btheta_{\rm pop};\alpha)
\to 0.
\]
Since \(L\) is convex and \(\lambda>0\), the population objective is
\(\lambda\)-strongly convex, so its minimiser is unique. Hence the empirical minimiser converges to the population minimiser as $\kappa \rightarrow 0$:
\[
\widehat{\btheta}_n
\to
\btheta_{\rm pop}(\alpha).
\]

\subsection{Variance decomposition and clean accuracy beyond the
eigenvector simplification}
\label{app:variance_decomp_general}
Section~\ref{sec:comparison} derived the variance
decomposition~\eqref{eq:variance_decomp} under the simplifying
assumption that $\bmu$ and $\bv$ are eigenvectors of~$\C$.  We show
here that the same structure persists for generic $\bmu,\bv$
satisfying Assumption~\ref{ass:orthogonality}, so the monotone
clean-accuracy conclusion does not rely on that simplification.

Since $\mathbf{R} = (\lambda\I+\tau\C)^{-1}$ commutes with $\C$, we
have $\mathbf{R}^2\C = \mathbf{R}\C\mathbf{R}$, and Theorem~\ref{thm:loureiro}
gives
\[
\sigma^2
= (\eta_1\bmu_1+\eta_2\bmu_2)^\top\mathbf{R}\C\mathbf{R}(\eta_1\bmu_1+\eta_2\bmu_2)
+ \zeta,
\qquad
\zeta := \frac{\gamma}{n}\tr[\mathbf{R}^2\C^2].
\]
Writing
$\eta_1\bmu_1+\eta_2\bmu_2 = (\eta_1-\eta_2)\bmu + \eta_2\alpha\bv$ and
expanding,
\[
\sigma^2
= (\eta_1-\eta_2)^2\,\bmu^\top\mathbf{R}\C\mathbf{R}\bmu
+ 2(\eta_1-\eta_2)\eta_2\alpha\,\bmu^\top\mathbf{R}\C\mathbf{R}\bv
+ \eta_2^2\alpha^2\,\bv^\top\mathbf{R}\C\mathbf{R}\bv
+ \zeta.
\]
The matrix $\mathbf{R}\C\mathbf{R} = \C(\lambda\I+\tau\C)^{-2}$ is a
bounded function of~$\C$, so by a Weierstrass-polynomial approximation
argument (cf.\ Remark~\ref{rem:krylov}) Assumption~\ref{ass:orthogonality}
gives $\bmu^\top\mathbf{R}\C\mathbf{R}\bv = o(1)$, and the cross term
vanishes.
To re-express in terms of the mean alignments
$a := \EE[\bmu^\top\tilde\btheta]$ and
$b := \EE[\bv^\top\tilde\btheta]$, recall
that (again up to $o(1)$ from $\bmu^\top\mathbf{R}\bv = o(1)$)
$a = (\eta_1-\eta_2)\,\bmu^\top\mathbf{R}\bmu$ and
$b = \eta_2\alpha\,\bv^\top\mathbf{R}\bv$.  Since $\lambda>0$, both
$\bmu^\top\mathbf{R}\bmu$ and $\bv^\top\mathbf{R}\bv$ are strictly
positive, so the substitutions
$(\eta_1-\eta_2) = a/(\bmu^\top\mathbf{R}\bmu) + o(1)$ and
$\eta_2\alpha = b/(\bv^\top\mathbf{R}\bv) + o(1)$ are well-defined
as algebraic identities.  We caution that a priori $\bmu^\top\mathbf{R}\bmu$
and $\bv^\top\mathbf{R}\bv$ could vanish along the sequence if the
fixed-point parameter~$\tau$ grows unboundedly, in which case the
$o(1)$ residuals propagating below need not be small; the decomposition
is quantitatively informative in the regime where $\tau$ remains
bounded.  Assuming this, we substitute:
\begin{equation}
\label{eq:sigma2_general}
\sigma^2
\;=\;
\beta_\mu\,a^2
\;+\;
\beta_v\,b^2
\;+\;
\zeta
\;+\; o(1),
\qquad
\beta_\mu := \frac{\bmu^\top\mathbf{R}\C\mathbf{R}\bmu}
                   {(\bmu^\top\mathbf{R}\bmu)^2},\quad
\beta_v := \frac{\bv^\top\mathbf{R}\C\mathbf{R}\bv}
                 {(\bv^\top\mathbf{R}\bv)^2}.
\end{equation}
Both $\beta_\mu$ and $\beta_v$ are non-negative
($\bmu^\top\mathbf{R}\C\mathbf{R}\bmu = \|\C^{1/2}\mathbf{R}\bmu\|^2$),
and~\eqref{eq:sigma2_general} is the general-case analogue
of~\eqref{eq:variance_decomp}: the eigenvector simplification recovers
$\beta_\mu = \lambda_\mu/\|\bmu\|^2$ and $\beta_v = \lambda_\bv$.  In the
spectral basis of~$\C$, writing
$\bmu = \sum_i\mu_i\be_i$ with $\C\be_i = \lambda_i\be_i$,
\[
\beta_\mu
= \frac{\sum_i \mu_i^2\,\lambda_i/(\lambda+\tau\lambda_i)^2}
       {\bigl(\sum_i \mu_i^2/(\lambda+\tau\lambda_i)\bigr)^2},
\]
a resolvent-weighted average of the eigenvalues of~$\C$ along~$\bmu$,
and analogously for $\beta_\bv$.

\paragraph{Clean accuracy in the general case.}
Substituting~\eqref{eq:sigma2_general} into~\eqref{eq:clean_acc} with
$\btheta = \tilde\btheta$ (using the concentration
$\tilde\btheta^\top\C\tilde\btheta \to \sigma^2$) and folding the
small trigger contribution $\beta_v b^2$ (which vanishes as
$\alpha\to\infty$ by Proposition~\ref{prop:alignment_peaks} and is
empirically subdominant at all $\alpha$,
cf.\ Table~\ref{tab:sigma2_decomposition}) into an effective noise
$\tilde\zeta := \beta_v b^2 + \zeta$:
\begin{equation}
\label{eq:clean_acc_general_app}
\mathrm{Acc}_{\mathrm{clean}}
\;\approx\;
\Phi\!\left(\frac{a}{\sqrt{\beta_\mu\,a^2 + \tilde\zeta}}\right).
\end{equation}
This retains the structure used in Section~\ref{sec:comparison}:
\begin{itemize}
\item[(i)] \emph{Information limit} ($\kappa\to 0$, so $\zeta\to 0$
and, by Proposition~\ref{prop:pop_convergence}, $b\to 0$ at large
$\alpha$, giving $\tilde\zeta\to 0$).  The argument of $\Phi$
collapses to $1/\sqrt{\beta_\mu}$, independent of~$a$: clean accuracy
is flat in~$\alpha$.
\item[(ii)] \emph{Proportional regime} ($\zeta>0$).  Differentiating,
\[
\frac{d}{da}\,\frac{a}{\sqrt{\beta_\mu a^2 + \tilde\zeta}}
= \frac{\tilde\zeta}{(\beta_\mu a^2 + \tilde\zeta)^{3/2}} > 0,
\]
so clean accuracy is strictly increasing in~$a$ whenever
$\tilde\zeta > 0$.  Combined with the monotonicity of~$a$ in~$\alpha$
(Proposition~\ref{prop:pop_convergence} at the population level),
clean accuracy inherits the same monotonicity.
\end{itemize}

Thus the qualitative claim of Section~\ref{sec:comparison} --- that the
finite-sample noise floor $\zeta$ is the mechanism coupling benign
alignment to clean accuracy --- is independent of the eigenvector
simplification.  The simplification enters only through the
interpretation of the scalar~$\beta_\mu$ as a single eigenvalue of~$\C$
rather than as the resolvent-weighted spectral average above.

%% file: appendix_experiments.tex
\section{Further Experimental Details}
\label{app:experimental_details}

\paragraph{Dataset and preprocessing.}
Except for the synthetic linear-regression experiments in
Figures~\ref{fig:lin_peak_benign} and~\ref{fig:lin_eigen_story},all experiments use CIFAR-10 restricted to classes~0 (``airplane'') and~1 (``automobile''), pooling the original 50{,}000 training and 10{,}000 test images (12{,}000 total per class pair --- 6{,}000 per class) and applying a single random 80/20 train/test split (seed~42), yielding $n_{\mathrm{train}}\approx 9{,}600$ and $n_{\mathrm{test}}\approx 2{,}400$ samples. Pixel intensities are normalised to $[0,1]$ (divided by 255) and the training set mean is subtracted from both splits. The feature dimension is $d = 32 \times 32 \times 3 = 3{,}072$.

\paragraph{Poisoning protocol.}
For the CIFAR-10, Gaussian-surrogate, and ResNet experiments, a fraction
\(\phi=0.05\) of the dataset with class~\(-1\) airplane training samples is selected
uniformly at random and poisoned: the trigger vector \(\alpha\bv\), with
\(\|\bv\|=1\), is added to each selected sample, and its label is flipped to
\(+1\). The trigger direction is a \(2{\times}2\) corner patch in the top-left
of the image applied uniformly across all three RGB channels, normalised to a
unit vector and then scaled to norm~\(\alpha\).

For these image experiments, we sweep \(\alpha\) over 20 equally-spaced values
in \([0,1.5]\). Attack success rate is always evaluated at a fixed test trigger
norm \(\alpha_{\mathrm{test}}=0.5\) in the same direction as the training
trigger. The synthetic linear-regression experiments in
Figures~\ref{fig:lin_peak_benign} and~\ref{fig:lin_eigen_story} use the
synthetic Gaussian-mixture setup and poisoning parameters specified in the
corresponding figure-specific paragraphs below.

\paragraph{Gaussian surrogate construction.}
The Gaussian surrogate dataset is constructed to match the empirical first- and
second-order statistics of the real CIFAR-10 training split.  From the centred
training data we estimate the class mean $\hat\bmu = \frac{1}{n_{+}}\sum_{i:\,y_i=+1}\x_i$
and the pooled within-class covariance
$\hat\C = \tfrac{1}{2}\hat\Sigma_{+1} + \tfrac{1}{2}\hat\Sigma_{-1}$,
regularised by a small ridge $\epsilon\I$ ($\epsilon = 10^{-4}\,\mathrm{tr}(\hat\C)/p$).
Synthetic samples are then drawn as
$\x \sim \mathcal{N}(\pm\hat\bmu,\,\hat\C)$ via the Cholesky factorisation
$\hat\C = \mathbf{L}\mathbf{L}^\top$, using $\x = \pm\hat\bmu + \mathbf{L}\mathbf{z}$
with $\mathbf{z} \sim \mathcal{N}(\mathbf{0}, \mathbf{I})$.
The surrogate dataset uses the same $n_{\mathrm{train}}$ and $n_{\mathrm{test}}$
as the corresponding real split and balanced class sizes
($\lfloor(n_{\mathrm{train}}+n_{\mathrm{test}})/2\rfloor$ samples per class).
All other experimental choices---poisoning fraction, trigger direction,
regularisation strength, and number of seeds---are identical to those of the
real CIFAR-10 experiment.

\paragraph{Figure~\ref{fig:resnet_intro}: ResNet-18 triple panel.}
We train a ResNet-18~\citep{heDeepResidualLearning2015} with a single scalar output (binary classification). For CIFAR-10's $32{\times}32$ images we replace the standard $7{\times}7$/stride-2 first convolution with a $3{\times}3$/stride-1 convolution and remove the initial max-pool, following common practice for small images. Training uses SGD with momentum $0.9$, weight decay $5\times10^{-4}$, initial learning rate $0.05$ annealed to~$0$ via a cosine schedule over 50 epochs, and mini-batch size~128. The experiment is repeated for 25 independent train/test splits (seeds $42, 43, \ldots, 66$); shaded bands show mean~$\pm$ one standard deviation.

The three panels show: (left) clean test accuracy versus training trigger norm~$\alpha$; (centre) ASR versus~$\alpha$; (right) ASR versus the $\log_{10}$-eigenvalue of the pooled within-class covariance $\C$ of CIFAR-10. For the rightmost panel, 20 eigenvectors of~$\C$ are selected log-uniformly across the full eigenvalue spectrum; each eigenvector is used as the trigger direction with $\alpha = \alpha_{\mathrm{test}} = 0.5$ fixed. The pooled covariance is $\C = \frac{1}{2}\hat\Sigma_{+1} + \frac{1}{2}\hat\Sigma_{-1}$ with a small ridge $\epsilon\I$ ($\epsilon = 10^{-4}\,\mathrm{tr}(\C)/d$) for numerical stability, eigendecomposed via \texttt{numpy.linalg.eigh}. Each eigenvector point is averaged over 25 independent runs.

\paragraph{Figure~\ref{fig:clean_acc_cifar}: CIFAR-10 vs.\ Gaussian empirical (logistic regression).}
Logistic regression with $\ell_2$ regularisation strength $\lambda = 10^{-4}$ is trained on (i)~the real CIFAR-10 split and (ii)~a synthetic Gaussian dataset whose class-conditional means and covariance are estimated empirically from the same CIFAR-10 training split; both use the same $n_{\mathrm{train}}$ and $n_{\mathrm{test}}$. The classifier is obtained by minimising
\[
  \frac{\lambda}{2}\|\btheta\|^2 + \frac{1}{n}\sum_{i=1}^n \log\bigl(1+e^{-y_i(\btheta^\top\x_i)}\bigr)
\]
via PyTorch L-BFGS with the strong Wolfe line-search condition (up to 2{,}000 iterations). The experiment is repeated for 25 independent seeds; shaded bands show mean~$\pm$ one standard deviation. The four panels display clean test accuracy, ASR, alignment $\btheta^\top\bmu$, and alignment $\btheta^\top\bv$ as functions of training trigger norm~$\alpha$.

\paragraph{Figure~\ref{fig:alpha_test_sensitivity}: sensitivity to test trigger norm.}
The attack success rate in Figures~\ref{fig:resnet_intro} and~\ref{fig:clean_acc_cifar}
is reported at a fixed test trigger norm $\alpha_{\mathrm{test}}=0.5$. To check that
the qualitative ASR-vs-$\alpha$ shape is not an artifact of this choice, we repeat
both experiments at additional values of $\alpha_{\mathrm{test}}$. Panel~(a) shows
the logistic-regression sweep on CIFAR-10 and panel~(b) shows the ResNet-18 sweep,
both with all other settings identical to the main-text figures. We additionally connect the points where $\alpha = \alpha_{\text{test}}$, and label the maximum value of each curve with a star. 

\begin{figure}[ht]
    \centering
    \begin{subfigure}[b]{0.48\textwidth}
        \centering
        \includegraphics[width=\textwidth]{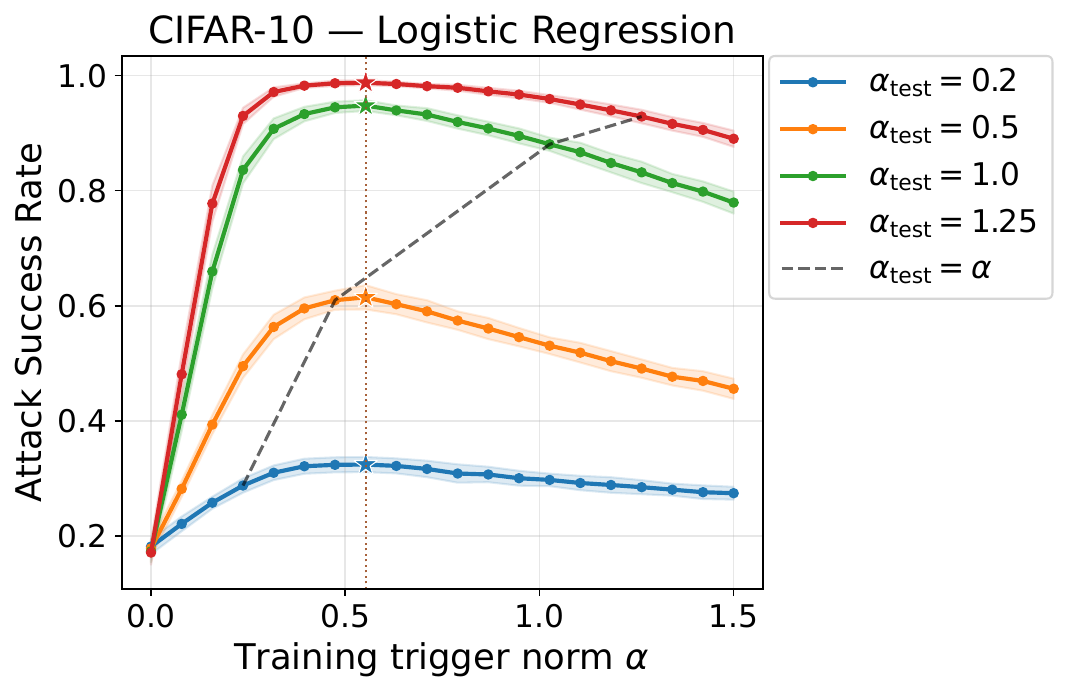}
        \caption{Logistic regression, CIFAR-10.}
    \end{subfigure}
    \hfill
    \begin{subfigure}[b]{0.48\textwidth}
        \centering
        \includegraphics[width=\textwidth]{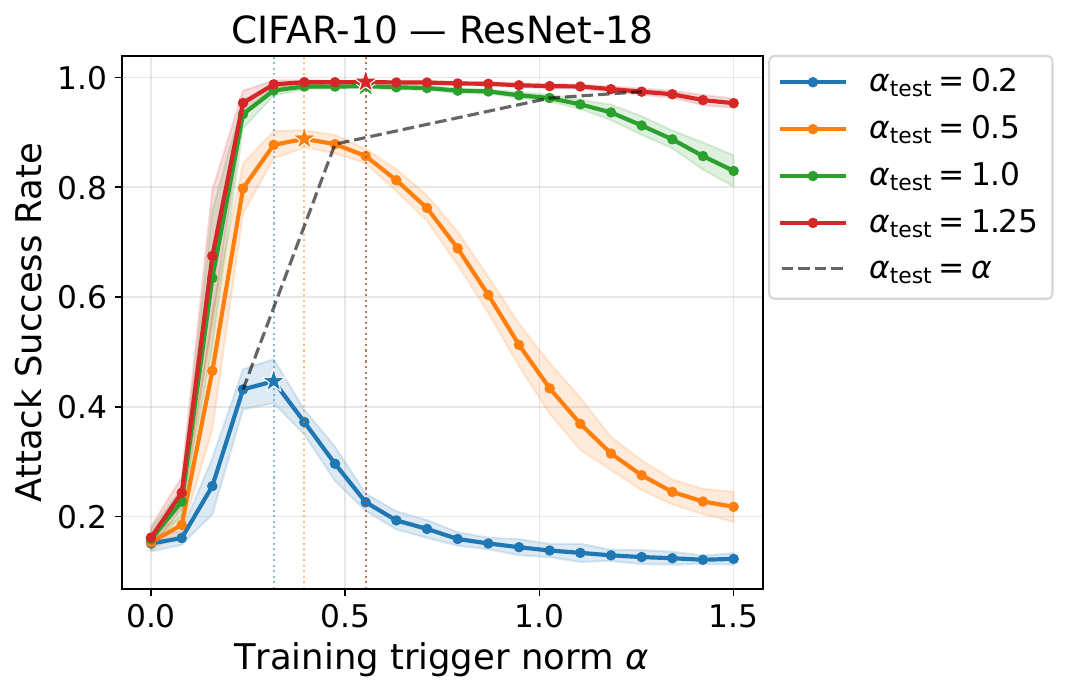}
        \caption{ResNet-18, CIFAR-10.}
    \end{subfigure}
    \caption{ASR vs.\ training trigger strength $\alpha$ at several values of the
    test trigger norm $\alpha_{\mathrm{test}}$.}
    \label{fig:alpha_test_sensitivity}
\end{figure}

\paragraph{Figure~\ref{fig:lin_peak_benign}: projection curves across aspect ratios.}
This figure is a theory-only comparison across several proportional-regime
aspect ratios \(\kappa=p/n\), including the overparameterized regime
\(\kappa>1\). We solve the deterministic theory on the grid
\(\alpha\in[0,30]\) for the four configurations
\[
(p,n,\kappa)\in
\{(1100,1000,1.10),\ (1050,1000,1.05),\ (1000,2000,0.50),\
(1000,5000,0.20)\}.
\]
For each value of \(\kappa\), the dashed curves show the squared-loss
prediction and the solid curves show the logistic-loss fixed-point prediction.
The left panel plots the trigger projection
\(|\langle\theta^\star,v\rangle|\), which rises for small \(\alpha\), reaches a
finite peak, and then decays. The right panel plots the benign projection
\(\langle\theta^\star,\mu\rangle\), which increases with \(\alpha\). The main
purpose of the figure is to show that these qualitative behaviors are stable
across different values of \(p/n\), including the overparameterized regime
\(\kappa>1\).

\paragraph{Figure~\ref{fig:lin_eigen_story}: eigenvector spectral effect.}
This figure isolates the dependence of the trigger projection on the trigger
eigenvalue \(s_v^2\). We use a synthetic non-isotropic covariance model in
which the clean and trigger directions are orthogonal eigenvectors of \(\C\).
Writing
\[
u_\mu:=\frac{\bmu}{\|\bmu\|},
\]
the covariance used in the plotted experiment is
\[
\C
=
s_{\mathrm{rest}}^2
\bigl(\I-u_\mu u_\mu^\top-\bv\bv^\top\bigr)
+
s_\mu^2u_\mu u_\mu^\top
+
s_v^2\bv\bv^\top .
\]
Thus \(u_\mu\) has eigenvalue \(s_\mu^2\), the trigger direction \(\bv\) has
eigenvalue \(s_v^2\), and every direction orthogonal to
\(\operatorname{span}\{u_\mu,\bv\}\) has eigenvalue
\(s_{\mathrm{rest}}^2\). In the plotted setting, \(\|\bmu\|=1\),
\(\|\bv\|=1\), and \(s_{\mathrm{rest}}^2=1\), so the covariance reduces to the
rank-two perturbation
\[
\C
=
\I
+
(s_\mu^2-1)\bmu\bmu^\top
+
(s_v^2-1)\bv\bv^\top .
\]

The eigenvector specialization in Corollary~\ref{cor:lin_peak_eigen} is more
general than this particular plotted construction. It only requires that
\(\bmu\) and \(\bv\) are orthogonal eigenvectors of \(\C\), with eigenvalues
\(s_\mu^2\) and \(s_v^2\), respectively. The covariance on the orthogonal
complement of \(\operatorname{span}\{\bmu,\bv\}\) may be arbitrary positive
definite, subject to the standing boundedness assumptions; in the displayed
projection formulas, its effect enters through the fixed-point scalar \(\tau\).

The parameters used in the figure are
\[
p=300,
\qquad
n=5000,
\qquad
\kappa=0.06,
\qquad
\|\bmu\|=1,
\qquad
\phi=0.20,
\qquad
\lambda=0.50,
\]
with \(s_\mu^2=2.0\) and \(s_{\mathrm{rest}}^2=1.0\). We vary only the trigger
eigenvalue \(s_v^2\), keeping \(\bv\) a unit-norm covariance eigenvector
orthogonal to \(\bmu\). The left panel plots exact trigger-projection curves
\(h_v(\alpha)\) for
\[
s_v^2\in\{0.20,0.35,0.80,1.80\},
\]
over the grid
\[
\alpha\in\{0,0.5,1.0,\ldots,8.0\}.
\]
The middle panel plots the peak value \(\max_\alpha h_v(\alpha)\) for
\[
s_v^2\in\{0.20,0.35,0.50,0.80,1.00,1.40,1.80\}.
\]
The right panel fixes the training trigger strength at \(\alpha=4.0\) and
compares the exact theoretical prediction for \(h_v(4.0)\) with finite-sample
ridge estimates \(\bv^\top\hat{\btheta}\). The finite-sample sweep is repeated
over \(8\) independent repetitions with seed \(2027\). Markers show empirical
means, and error bars show the standard error of the mean.

These panels illustrate Corollary~\ref{cor:lin_min_eig}: within the
covariance-eigenvector class, increasing \(s_v^2\) increases the denominator of
the trigger-projection formula through
\[
B_v=\lambda+\tau s_v^2,
\]
and therefore attenuates the learned trigger projection.

\paragraph{Figure~\ref{fig:lin_phi_sweep_appendix}: poisoning-fraction sweep.}
This figure shows how the square-loss projections vary with the poisoning
fraction \(\phi\). We use the same synthetic Gaussian-mixture setup as in the
linear square-loss experiments, and sweep
\[
\phi\in\{0,0.05,0.1,0.2,0.3,0.4,0.5\}
\]
and
\[
\alpha\in\{0,0.5,1.0,\ldots,10.0\}.
\]
For each pair \((\phi,\alpha)\), we generate a finite training sample from the
absorbed two-component Gaussian mixture, fit the squared-loss ridge estimator
\(\hat{\btheta}\), and record the empirical projections
\[
\bmu^\top\hat{\btheta},
\qquad
\bv^\top\hat{\btheta}.
\]
We also compute the corresponding square-loss theory predictions from the
closed-form linear formulas. Solid curves show empirical ridge estimates and
dashed curves show the theory predictions.

The left panel plots the trigger alignment \(h_v(\alpha)\), while the right
panel plots the benign alignment \(h_\mu(\alpha)\). Increasing \(\phi\) has two
different effects. For the trigger alignment, larger poisoning fractions can
increase \(h_v\) at small trigger strengths, because more poisoned examples
reinforce the trigger direction. At larger trigger strengths, the denominator
effect in the square-loss formula dominates, and increasing \(\phi\) can reduce
the trigger alignment. In contrast, the benign alignment decreases with
\(\phi\) at fixed \(\alpha\), to leading order. The curves \(\phi=0\) and
\(\phi=1/2\) are included as visual reference cases; the formal statements in
the main text assume \(0<\phi<1/2\).

\paragraph{Table~\ref{tab:sigma2_decomposition}: fixed-point solution.}
The table reports the asymptotic prediction for $\sigma^2 = \mathbb{E}[\tilde\btheta^\top\C\tilde\btheta]$,
the variance of the classifier's output on a test point, derived from a high-dimensional analysis
of regularised logistic regression. In the limit $p, n \to \infty$ with aspect ratio
$p/n \to \kappa$, the trained weight vector concentrates around a Gaussian
proxy $\tilde\btheta$ with average $(\lambda\I + \tau\C)^{-1}(\eta_1\bmu + \eta_2\bv)$, where the scalar
parameters $(\tau, \gamma, \delta, \eta_1, \eta_2)$ are the unique solution to a system of
self-consistent fixed-point equations that couple the loss curvature, effective noise variance,
and class-conditional signals.

These equations are solved numerically using the empirically estimated class mean $\bmu$ and
pooled within-class covariance $\C$ from CIFAR-10 (classes 0 and 1), with $\lambda = 10^{-4}$,
poisoning fraction $\phi = 0.1$, and aspect ratio $p/n \approx 0.32$. The expectation integrals
that appear in the fixed-point equations are evaluated by Gauss--Hermite quadrature (100 nodes),
and the system is iterated to convergence (residual $< 10^{-10}$).

At the solution, $\sigma^2$ decomposes as
\begin{align}
\sigma^2
&= \underbrace{(\eta_1-\eta_2)^2\,\bmu^\top\mathbf{A}\bmu}_{\text{mean-direction signal}}
 + \underbrace{\eta_2^2\,\bv^\top\mathbf{A}\bv}_{\text{trigger-direction signal}} \notag \\
&\quad + \underbrace{2(\eta_1-\eta_2)\eta_2\,\bmu^\top\mathbf{A}\bv}_{\text{cross term}}
 + \underbrace{\tfrac{\gamma}{n}\mathrm{tr}\bigl[(\lambda\I+\tau\C)^{-2}\C^2\bigr]}_{\text{fundamental noise }
\zeta},
\end{align}
where $\mathbf{A} := (\lambda\I+\tau\C)^{-2}\C$. The first three terms reflect how much of the
classifier's variance is driven by alignment with the clean-class mean and the trigger direction
respectively; the last term is an irreducible noise floor arising from finite-sample fluctuations,
present even in the absence of poisoning.

\paragraph{Compute resources.}
All experiments were run on a single workstation with a 32-core (64-thread) CPU and four NVIDIA RTX~3080~Ti GPUs. The full ResNet-18 sweep (Figure~\ref{fig:resnet_intro}: $3$ panels $\times\,20$ trigger configurations $\times\,25$ seeds, each a 50-epoch training run) took approximately 14 hours when distributed across the four GPUs. The logistic-regression and Gaussian-surrogate sweeps (Figure~\ref{fig:clean_acc_cifar}) ran on CPU and completed in well under an hour. Fixed-point solves used to produce the theoretical curves and Table~\ref{tab:sigma2_decomposition} take a few seconds per $(\alpha,\kappa)$ configuration.

\paragraph{Licensing of existing assets.}
CIFAR-10~\citep{krizhevsky2009learning} is publicly distributed by the University of Toronto for research use; the ResNet architecture is from~\citet{heDeepResidualLearning2015} and our implementation uses the \texttt{torchvision} reference (BSD-3-Clause). PyTorch (BSD-3-Clause), NumPy (BSD-3-Clause), and SciPy (BSD-3-Clause) were used for training and the fixed-point solver.